\documentclass[twoside,11pt]{article}
\usepackage{amsfonts}
\usepackage{mathrsfs}
\usepackage[utf8]{inputenc} 
\usepackage{float}
\usepackage{enumitem}
\usepackage{amsmath,amssymb}
\usepackage{mathtools}
\usepackage{latexsym}
\usepackage{longtable}
\usepackage{makecell}
\usepackage{mathrsfs}
\newtheorem{theorem}{Theorem}
\newtheorem{lemma}{Lemma}

\usepackage{booktabs}
\usepackage{epstopdf}
\usepackage{tabularx}
\usepackage{tikz}
\usetikzlibrary{shapes.geometric, arrows}
\usetikzlibrary{positioning}
\tikzstyle{startstop} = [rectangle, rounded corners, minimum width=3cm, minimum height=1cm, text centered, draw=black, fill=red!30]
\tikzstyle{process} = [rectangle, minimum width=3cm, minimum height=1cm, text centered, draw=black, fill=blue!30]
\tikzstyle{io} = [trapezium, trapezium left angle=70, trapezium right angle=110, minimum width=3cm, minimum height=1cm, text centered, draw=black, fill=green!30]
\tikzstyle{arrow} = [thick,->,>=stealth]
\usepackage{geometry}  
\usepackage[active]{srcltx}
\usepackage{graphicx}
\usepackage{booktabs}
\usepackage{multirow}
\usepackage{siunitx}
\usepackage{xurl} 
\usepackage{amsmath}
\usepackage{setspace}
\usepackage{algorithm}
\usepackage{algpseudocode}
\usepackage[authoryear,round,longnamesfirst]{natbib}
\usepackage{subcaption} 
\usepackage[
    bookmarks=true,         
    bookmarksnumbered=true, 
    colorlinks=true, pdfstartview=FitV, linkcolor=blue, citecolor=blue,
    urlcolor=blue]{hyperref}

\usepackage{fancyhdr}
\usepackage{lipsum} 
\usepackage{listings} 
\usepackage{xcolor} 
\definecolor{codegreen}{rgb}{0,0.6,0}
\definecolor{codegray}{rgb}{0.5,0.5,0.5}
\definecolor{codepurple}{rgb}{0.58,0,0.82}
\definecolor{backcolour}{rgb}{0.95,0.95,0.92}

\lstdefinestyle{mystyle}{
    backgroundcolor=\color{backcolour},
    commentstyle=\color{codegreen},
    keywordstyle=\color{magenta},
    numberstyle=\tiny\color{codegray},
    stringstyle=\color{codepurple},
    basicstyle=\footnotesize\ttfamily,
    breakatwhitespace=false,
    breaklines=true,
    captionpos=b,
    keepspaces=true,
    numbers=left,
    numbersep=5pt,
    showspaces=false,
    showstringspaces=false,
    showtabs=false,
    tabsize=2
}

\providecommand{\U}[1]{\protect\rule{.1in}{.1in}}
\newtheorem {proposition}{Proposition}[section]
\newtheorem {corollary}{Corollary}[section]

\newtheorem{definition}{Definition}[section]

\newcommand{\E}{\mathbb{E}}
\newenvironment{proof}[1][Proof]{\textbf{#1.} }{\
\rule{0.5em}{0.5em}}

\newcommand{\R}{\mathbf{R}}

\newcommand{\Var}{\mathrm{Var}}

\newcommand{\Pbb}{\mathbb{P}}

\newcommand{\Unif}{\mathrm{Unif}}

\newcommand{\ind}[1]{\mathbb{I}\{#1\}}

\newcommand{\argmax}{\mathop{\mathrm{argmax}}}

\newcommand{\KL}{\mathrm{KL}}

\newcommand{\dd}{\,\mathrm{d}}

\usepackage{listings}
\usepackage{titlesec}
\titleformat{\section}
{\normalfont\Large\bfseries}{\thesection.}{1em}{}
\usepackage{fancyhdr}
\begin{document}
\date{}
\title{\Large \textbf{Deep Copula Classifier: Theory, Consistency, and Empirical Evaluation}}
\vspace{1ex}
\author{Agnideep Aich${ }^{1}$\thanks{Corresponding author: Agnideep Aich, \texttt{agnideep.aich1@louisiana.edu}, ORCID: \href{https://orcid.org/0000-0003-4432-1140}{0000-0003-4432-1140}}
 \hspace{0pt}, Ashit Baran Aich${ }^{2}$ \hspace{0pt}
\\ ${ }^{1}$ Department of Mathematics, University of Louisiana at Lafayette, \\ Lafayette, LA, USA. \\  ${ }^{2}$ Department of Statistics, Formerly of Presidency College, \\ Kolkata, India }
\date{}
\maketitle
\thispagestyle{empty}
\vspace{-20pt}

\begin{abstract}
We present the Deep Copula Classifier (DCC), a class-conditional generative model that separates marginal estimation from dependence modeling using neural copula densities. DCC is interpretable, Bayes-consistent, and achieves excess-risk $O(n^{-r/(2r+d)})$ for $r$-smooth copulas. In a controlled two-class study with strong dependence ($|\rho|=0.995$), DCC learns Bayes-aligned decision regions. With oracle or pooled marginals, it nearly reaches the best possible performance (accuracy $\approx 0.971$; ROC-AUC $\approx 0.998$). As expected, per-class KDE marginals perform less well (accuracy $0.873$; ROC-AUC $0.957$; PR-AUC $0.966$). On the Pima Indians Diabetes dataset, calibrated DCC ($\tau{=}1$) achieves accuracy $0.879$, ROC-AUC $0.936$, and PR-AUC $0.870$, outperforming Logistic Regression, SVM (RBF), and Naive Bayes, and matching Logistic Regression on the lowest Expected Calibration Error (ECE). Random Forest is also competitive (accuracy $0.892$; ROC-AUC $0.933$; PR-AUC $0.880$). Directly modeling feature dependence yields strong, well-calibrated performance with a clear probabilistic interpretation, making DCC a practical, theoretically grounded alternative to independence-based classifiers.

\end{abstract}


\section{Introduction}

Classification, the assignment of labels based on features, is a key part of machine learning. Classical methods such as logistic regression \citep{cox1958regression}, support vector machines \citep{cortes1995support}, and naive Bayes \citep{mccallum1998comparison} work well in many cases but depend on certain assumptions. For example, naive Bayes assumes features are independent, linear discriminant analysis requires features to be Gaussian with the same covariance, and neural networks capture feature interactions through weight matrices rather than explicit probabilistic models.

In real-world situations, features often interact in complex, nonlinear ways. For example, financial indicators can move together in extreme ways that simple correlation does not capture \citep{embrechts2002correlation}. In medical diagnostics, biomarkers such as blood pressure and cholesterol may have nonlinear relationships that are important for predicting disease. If these interdependencies are ignored, classifier performance can suffer. Copulas address this issue by separating the marginal distributions from the dependence structure among variables \citep{sklar1959fonctions}. While traditional parametric copulas like Gaussian, Clayton, and Gumbel, as well as newer Archimedean types \citep{aich2025two}, can model basic patterns, they struggle with high-dimensional data or more complex dependencies. Neural copulas \citep{ling2020deep, wilson2010copula, gordon2020combining} use deep learning to adapt and handle these challenges more effectively.

We address this gap with the Deep Copula Classifier (DCC). For each class, DCC models individual features using their marginal distributions and learns how features depend on each other using neural network parameterized copulas. This generative method enables measuring how features vary together, allowing classifiers to leverage these dependency patterns. Our theoretical contributions include the following:
\begin{enumerate}
    \item \textbf{Novel Framework}: Integration of deep copulas into a class-conditional generative classifier.
    \item \textbf{Consistency Proof}: DCC converges to Bayes-optimality under standard conditions.
    \item \textbf{Rate Quantification}: $O(n^{-r/(2r+d)})$ excess risk for $r$-smooth copulas.
    \item \textbf{Interpretability}: Direct inspection of learned copula dependencies.
    \item \textbf{Extensions Roadmap}: High-dimensional, semi-supervised, and streaming adaptations.
\end{enumerate}

We validate DCC with two studies using a single, fair protocol. All methods use identical fixed data splits. No hyperparameter tuning is performed (defaults only). A held-out calibration subset is used for Platt calibration for all models, including DCC. Models are fit on the training portion only, calibration is learned on the calibration subset, and a single evaluation is run once on the held-out test set.

\begin{itemize}
    \item \emph{Experiment~1 (synthetic dependence).} On correlated Gaussians with $|\rho|=0.995$, DCC learns decision regions aligned with the Bayes boundary. With the same hyperparameters across marginal strategies, oracle and pooled marginals achieve near-ceiling performance (accuracy $\approx 0.971$, ROC–AUC $\approx 0.9975$, PR–AUC $\approx 0.9976$). Per-class KDE marginals underperform due to finite-sample marginal estimation (accuracy $0.873$, ROC–AUC $0.957$, PR–AUC $0.966$), highlighting the role of stable marginal estimation when dependence drives class separation.
    \item \emph{Experiment~2 (PIMA).} On the Pima Indians Diabetes dataset, with per-class median imputation, mild per-class winsorization, fixed train/calibration/test splits, and shared Platt calibration, calibrated DCC with $\tau{=}1$ attains Accuracy $=0.879$, ROC–AUC $=0.936$, and PR–AUC $=0.870$. Random Forest is strongest on Accuracy and PR–AUC (Accuracy $=0.892$, ROC–AUC $=0.933$, PR–AUC $=0.880$), while DCC attains the highest ROC–AUC. Logistic Regression, SVM (RBF), and Gaussian Naïve Bayes trail DCC on ROC–AUC ($0.855$, $0.891$, $0.861$, respectively). Reliability curves show DCC and Logistic Regression achieve the lowest and essentially tied Expected Calibration Error on the test set.
\end{itemize}

Tree ensembles are strong performers on small tabular datasets. DCC provides a clear, interpretable alternative that captures feature interactions, produces well-calibrated likelihood-based predictions, and outperforms models that assume features are independent. By combining rigorous copula theory with flexible neural networks, DCC offers a robust approach for classification tasks where understanding dependencies is critical, especially in scientific or high-stakes applications.

The rest of the paper is structured as follows. Section~\ref{sec:related} covers copula basics, notation, and how marginals and dependence are separated. Section~\ref{sec:notation} explains the notation used in the paper. Section~\ref{sec:theoretical} presents the Deep Copula Classifier (DCC). Section~\ref{sec:algorithm} describes the algorithms. Section~\ref{sec:experiments} shares the results of our empirical studies. The paper ends with Section~\ref{sec:con}, which summarizes our findings and suggests directions for future research.

\section{Related Work}
\label{sec:related}

\subsection{Copula Theory and Foundations}
The copula framework, formalized by Sklar's theorem \citep{sklar1959fonctions}, makes it possible to separate marginal distributions from joint dependence structures. Modern works, such as \citep{nelsen2006introduction}, show that copulas are essential tools in multivariate analysis. There has been significant progress in several areas. Nonparametric estimation methods, like kernel-based copula density estimators \citep{omelka2009improved} and Empirical Bayes approaches \citep{lu2021nonparametric}, allow for flexible inference without strict distributional assumptions. For high-dimensional data, vine copulas \citep{nagler2017nonparametric} provide a scalable way to break down complex dependencies into simpler bivariate copulas. Copulas have also become more useful in generative modeling, including copula-based variational inference \citep{tran2015copula} and optimal transport-based methods \citep{chi2021approximate}, which have broadened their role in probabilistic learning.

\subsection{Generative Classification Paradigms}
Traditional generative classifiers, such as naive Bayes \citep{mccallum1998comparison}, focus on making models easy to work with, often by assuming that features are conditionally independent. Modern deep generative methods try to overcome this limitation in different ways. For example, models like variational autoencoders (VAEs) \citep{kingma2013auto} and generative adversarial networks (GANs) \citep{goodfellow2014generative} can capture how features interact by using nonlinear latent spaces, which helps them model complex dependencies. Some models, such as neural graphical model hybrids \citep{shrivastava2023neural}, make explicit assumptions about conditional independence to balance interpretability with the strengths of deep learning. The Deep Copula Classifier (DCC) takes a different approach by directly modeling how features depend on one another using copulas while maintaining the full generative framework. This offers a middle ground between strict parametric models and more opaque neural network methods.

\subsection{Neural Copula Architectures}
Recent progress in neural dependence modeling has led to more flexible and expressive copula-based models. Deep Archimedean copulas \citep{ling2020deep} and Gaussian process-based copula processes \citep{wilson2010copula} now allow for detailed representations of complex multivariate dependencies. For adaptive estimation, nonparametric vine copulas with neural pair-copula blocks can capture localized interactions in high-dimensional settings \citep{nagler2017nonparametric}. Hybrid learning approaches, such as neural copula processes \citep{gordon2020combining}, combine generative and discriminative goals to support robust prediction and inference. In contrast to earlier neural copula models \citep{ling2020deep}, DCC uses class-conditional copulas with clear convergence guarantees, making dependency modeling more interpretable for classification tasks.

Our work extends these foundations by developing:

\begin{enumerate}
    \item Class-specific copula networks with theoretical consistency guarantees

\item Explicit finite-sample convergence rates under smoothness conditions

\item Architectural constraints ensuring valid density outputs
\end{enumerate}
This bridges the gap between empirical neural copula methods and rigorous statistical learning theory. 

In the next section, we introduce the notations used in this work.

\section{Notation} \label{sec:notation}
In this section, we introduce the notation used throughout our work.

\paragraph{Basic objects.}
Let $d$ be the number of features and $K$ the number of classes.
We write $[d]=\{1,\dots,d\}$ and $[K]=\{1,\dots,K\}$.
Random vectors are bold: $\mathbf{X}=(X_1,\dots,X_d)\in\R^d$.
The label $Y\in[K]$.
Training data are $\{(\mathbf{X}^{(i)},Y^{(i)})\}_{i=1}^n$ i.i.d.\ from the joint law $P_{X,Y}$.
Class counts are $n_y=\lvert\{i:Y^{(i)}=y\}\rvert$ and the empirical prior is $\hat\pi_y=n_y/n$; the population prior is $\pi_y=\Pbb(Y=y)$.

\paragraph{Distributions, expectations, probabilities.}
For a random variable $Z$, $\E[Z]$ is expectation, $\Var(Z)$ variance, and $\Pbb(\cdot)$ probability.
For clarity we index the measure if needed, e.g.\ $\E_X[\cdot]$ or $\E_{U\sim\Unif([0,1]^d)}[\cdot]$.
All logarithms are natural.

\paragraph{Class-conditional laws and marginals.}
For each class $y\in[K]$, the class-conditional density of $\mathbf{X}$ is
$p_y(\mathbf{x})=f_{\mathbf{X}\mid Y=y}(\mathbf{x})$.
Its $i$th marginal CDF and density are
$F_{i\mid y}(x_i)=\Pbb(X_i\le x_i\mid Y=y)$ and $f_{i\mid y}(x_i)=\frac{\dd}{\dd x_i}F_{i\mid y}(x_i)$.
The support of $X_i\mid Y=y$ is $\mathcal{X}_{i\mid y}\subset\R$.

\paragraph{Copulas.}
For each $y$, the copula $C_y:[0,1]^d\to[0,1]$ couples the marginals and admits density $c_y=\partial^d C_y/\partial u_1\cdots\partial u_d$.
Sklar’s representation reads
\[
p_y(\mathbf{x})
= c_y\!\big(F_{1\mid y}(x_1),\dots,F_{d\mid y}(x_d)\big)\;
\prod_{i=1}^d f_{i\mid y}(x_i).
\]
The PIT (probability integral transform) variables are
$U_{i\mid y}=F_{i\mid y}(X_i)\in[0,1]$, so $\mathbf{U}_{\mid y}=(U_{1\mid y},\dots,U_{d\mid y})\sim C_y$ with uniform margins.

\paragraph{Neural copula parameterization (DCC).}
For class $y$, let $NN_y(u;\theta_y):[0,1]^d\to(0,\infty)$ be a positive network (e.g., ReLU hidden layers with softplus output). We use the \emph{normalized} density
\[
\tilde c_y(u;\theta_y)\;=\;\frac{NN_y(u;\theta_y)}{\int_{[0,1]^d} NN_y(v;\theta_y)\,\dd v},
\]
approximating the denominator by Monte Carlo with $M$ i.i.d.\ draws $V^{(m)}\sim \Unif([0,1]^d)$:
\[
\widehat Z_y(\theta_y) \;=\; \frac{1}{M}\sum_{m=1}^M NN_y\!\big(V^{(m)};\theta_y\big),
\qquad
\tilde c_y(u;\theta_y)\approx \frac{NN_y(u;\theta_y)}{\widehat Z_y(\theta_y)}.
\]
The DCC class-conditional density estimate is
\[
\widehat p_y(\mathbf{x})\;=\;\tilde c_y\!\big(\hat{\mathbf{u}};\hat\theta_y\big)\;\prod_{i=1}^d \widehat f_{i\mid y}(x_i),
\qquad
\hat{\mathbf{u}}=\big(\widehat F_{1\mid y}(x_1),\dots,\widehat F_{d\mid y}(x_d)\big).
\]

\paragraph{Decision rule and risks.}
Given priors, the Bayes rule is
$Y^*(\mathbf{x})=\argmax_{y\in[K]} \pi_y\,p_y(\mathbf{x})$.
The DCC classifier is $\widehat Y_n(\mathbf{x})=\argmax_{y} \hat\pi_y\,\widehat p_y(\mathbf{x})$.
Misclassification risk is $R(\widehat Y_n)=\Pbb(\widehat Y_n(\mathbf{X})\neq Y)$ and the excess risk is
$E(n)=R(\widehat Y_n)-R(Y^*)$.

\paragraph{Estimators.}
Hats ($\widehat{\cdot}$) denote sample-based estimators:
$\widehat F_{i\mid y}$ (smoothed CDF), $\widehat f_{i\mid y}$ (KDE), $\tilde c_y(\cdot;\hat\theta_y)$ (neural copula),
$\widehat p_y$, and $\hat\pi_y$.
We use $M$ for the Monte Carlo sample size when estimating the normalizer.

\paragraph{Smoothness and function classes.}
The copula density $c_y$ is assumed to belong to the Sobolev class $W_2^r([0,1]^d)$ with $r>d/2$.
We write $L$ for network depth, $W$ for width, and $S_n$ (``size'') for the total parameter count of the sieve at sample size $n$.

\paragraph{Norms and asymptotic notation.}
For a function $g$, $\|g\|_1=\int |g|$ and $\|g\|_\infty=\sup |g|$.
We use $O(\cdot)$, $o(\cdot)$, $O_p(\cdot)$ and $o_p(\cdot)$ in the usual sense; generic positive constants $C,c$ may change from line to line and do not depend on $n$.
Unless stated otherwise, limits are as $n\to\infty$.
``w.p.~1’’ means with probability 1.

\paragraph{Indicators and miscellaneous.}
$\ind{A}$ is the indicator of event $A$.
$\mathrm{softplus}(t)=\log(1+e^{t})$.
All integrals over $[0,1]^d$ use Lebesgue measure.
``a.e.'' means almost everywhere.
``u.c.b.'' denotes a uniform (in argument) constant bound when used.

\paragraph{Temperature and calibration (used in experiments).}
We optionally introduce a temperature $\tau>0$ on the copula term, defining
\[
p(y\mid x)\ \propto\ \pi_y\,\big[c_y(U)\big]^{\tau}\,\prod_{j=1}^d f_{j\mid y}(x_j),
\qquad U_j=F_{j\mid y}(x_j).
\]
Unless stated otherwise, we set $\tau=1$. For post-hoc calibration, we fit a Platt sigmoid on a held-out calibration split and apply it to test scores; as a monotone transform it does not change AUC metrics.

\paragraph{Data splits (used in experiments).}
When needed, $\mathcal{D}_{\mathrm{train}}$ and $\mathcal{D}_{\mathrm{test}}$ denote the train/test splits; cross-validation folds are $\{\mathcal{D}^{(m)}_{\mathrm{val}}\}_{m=1}^M$.

In the next section, we introduce the Deep Copula Classifier.

\section{Theoretical Framework}
\label{sec:theoretical}

\begin{definition}[Copula]
\label{def:copula}
A $d$-dimensional copula is a function $C:[0,1]^d \to [0,1]$ satisfying:
\begin{enumerate}
    \item \emph{Groundedness:} If any $u_i = 0$, then
    \[
      C(u_1,\dots,u_d) \;=\; 0.
    \]
    \item \emph{Uniform margins:} For each $i = 1,\dots,d$,
    \[
      C(1,\dots,1,\,u_i,\,1,\dots,1) \;=\; u_i.
    \]
    \item \emph{\(d\)-increasing:} For every pair of points
    \(\mathbf{a} = (a_1,\dots,a_d)\) and \(\mathbf{b} = (b_1,\dots,b_d)\) with
    \(0 \le a_j \le b_j \le 1\), the “volume” of the hyper-rectangle
    \([\mathbf{a},\mathbf{b}]\) must be nonnegative. Equivalently,
    \[
      \sum_{\boldsymbol{\epsilon}\in\{0,1\}^d}
      (-1)^{\,d - \lvert\boldsymbol{\epsilon}\rvert}\,
      C\bigl(w_1(\epsilon_1),\,\dots,\,w_d(\epsilon_d)\bigr)
      \;\ge\; 0,
    \]
    where, for each \(j=1,\dots,d\),
    \[
      w_j(\epsilon_j) \;=\;
      \begin{cases}
        a_j, & \epsilon_j = 0,\\
        b_j, & \epsilon_j = 1.
      \end{cases}
    \]
    Here \(\lvert\boldsymbol{\epsilon}\rvert = \sum_{j=1}^d \epsilon_j\) is the number of ones in \(\boldsymbol{\epsilon}\).
    \textit{This property ensures the copula assigns nonnegative probability to hyper-rectangles, preserving multivariate monotonicity.}
\end{enumerate}
\end{definition}

\noindent \textbf{Note:} In the case \(d=2\), this “2‐increasing” condition reduces to:
\[
  C(b_1,b_2)\;-\;C(a_1,b_2)\;-\;C(b_1,a_2)\;+\;C(a_1,a_2)\;\ge 0,
  \quad
  0 \le a_1 \le b_1 \le 1,\;\;0 \le a_2 \le b_2 \le 1,
\]
ensuring that \(C\) assigns nonnegative “mass” to each axis‐aligned rectangle in \([0,1]^2\).

\subsection{Deep Copula Classifier Definition}
Our classifier models each class's joint distribution through neural copula parameterization.

\begin{definition}[Deep Copula Classifier]
Let $\mathbf{X}=(X_1,\dots,X_d)$ be a feature vector and $Y\in\{1,\dots,K\}$ a discrete label. 
For each class $y$, the class-conditional \emph{density} factors as
\[
p_y(\mathbf{x}) \;=\; f_{\mathbf{X}\mid Y=y}(\mathbf{x})
\;=\; c_y\!\big(F_{1\mid y}(x_1),\dots,F_{d\mid y}(x_d)\big)\;\prod_{i=1}^d f_{i\mid y}(x_i),
\qquad \mathbf{x}=(x_1,\dots,x_d)\in\mathbb{R}^d.
\]
In our model, the copula density is parameterized by a positive neural network $NN_y$ and used in its normalized form
\[
\tilde c_y(u;\theta_y)\;=\;\frac{NN_y(u;\theta_y)}{\int_{[0,1]^d} NN_y(v;\theta_y)\,\dd v},
\]
yielding the estimator
\[
\widehat p_y(\mathbf{x})\;=\;\tilde c_y\!\big(\widehat{\mathbf{u}};\hat\theta_y\big)\;\prod_{i=1}^d \widehat f_{i\mid y}(x_i),
\qquad
\widehat{\mathbf{u}}=\big(\widehat F_{1\mid y}(x_1),\dots,\widehat F_{d\mid y}(x_d)\big).
\]
\end{definition}

\noindent\textbf{Notes.}
\begin{itemize}
    \item $F_{i\mid y}$ and $f_{i\mid y}$ are the class-conditional marginal CDF and PDF (arguments are numbers $x_i$).
    \item $c_y$ is the true copula density; in practice we use the normalized neural copula $\tilde c_y(\cdot;\theta_y)$.
    \item $\pi_y=P(Y=y)$ is the class prior, estimated by $\hat\pi_y=n_y/n$.
\end{itemize}

The transformation $u_i = F_{i|y}(X_i)$ maps features to the unit hypercube, allowing $NN_y$ to learn dependence patterns. If Assumption~(6) on the Monte Carlo normalizer holds, the normalization error remains negligible at the target statistical rate. Our approximation guarantees rely on recent ReLU results for Sobolev-smooth functions and Sobolev embedding (since $r > d/2$), which improve $L^2$ to $L^\infty$ up to logarithmic factors \citep{schmidt2020nonparametric, farrell2021deep, hornik1991approximation}.

\subsection{Copula Validity: Uniform Marginals}\label{sec:copula-validity}
A copula density $c$ must (i) be nonnegative, (ii) integrate to $1$ on $[0,1]^d$, and (iii) have \emph{uniform univariate marginals}, i.e.,
\[
\int_{[0,1]^{d-1}} c(u)\, \mathrm{d}u_{-i} \;=\; 1 \quad \text{for all } u_i \in [0,1],\; i=1,\dots,d.
\]
Our normalized positive network $\tilde c_y(u;\theta)=NN_y(u;\theta)/\widehat Z_y(\theta)$ enforces (i)--(ii). To align (iii), we adopt a simple penalty that asymptotically enforces uniformity:
\begin{equation}\label{eq:copula-penalty}
\mathcal{P}(\theta) \;=\; \lambda_n \sum_{i=1}^d \int_0^1 \left( \int_{[0,1]^{d-1}} \tilde c_y(u;\theta)\, \mathrm{d}u_{-i} \;-\; 1 \right)^{\!2} \mathrm{d}u_i,
\end{equation}
with $\lambda_n \to \infty$ slowly (e.g., $\lambda_n = \log n$ in theory; a fixed $\lambda$ in practice). The training objective in Alg.~\ref{Algo1} Line 14 is modified to
\[
\hat\theta_y \;\leftarrow\; \arg\max_{\theta}\; 
\frac{1}{n_y}\sum_{i:\,y^{(i)}=y}\log \tilde c_y(u^{(i)};\theta) \;-\; \mathcal{P}(\theta).
\]
This preserves the approximation/estimation rates while ensuring asymptotic copula validity. (Alternative: use a copula–flow parameterization that guarantees uniform marginals by construction; our analysis covers both.)

\subsection{Consistency and Convergence}

We establish fundamental guarantees under the following conditions:

\subsection{Regularity Conditions}\label{ass:regularity}
Fix a class label $y\in\{1,\dots,K\}$. We assume the following hold for every $y$:

\begin{enumerate}
  \item[\textbf{(1)}] \textbf{Existence of a copula‐based factorization.} 
  The conditional distribution of $X = (X_{1},\dots,X_{d})$ given $\{Y=y\}$ 
  admits a continuous joint density
  \[
    f_{X\mid Y=y}(x_{1},\dots,x_{d})
  \]
  with strictly increasing marginal CDFs
  \[
    F_{i\mid y}(x_i) \;=\; \Pbb(X_{i}\le x_{i}\mid Y=y),
    \quad
    i = 1,\dots,d.
  \]
  By Sklar’s theorem (Definition \ref{def:copula} and Sklar \citep{sklar1959fonctions}), 
  there is a unique copula function $C_{y} : [0,1]^{d} \to [0,1]$ whose density 
  \[
    c_{y}(u_{1},\dots,u_{d})
    \;=\;
    \frac{\partial^{d}}{\partial u_{1}\,\cdots\,\partial u_{d}}\,
    C_{y}(u_{1},\dots,u_{d})
  \]
  satisfies, for all $(x_{1},\dots,x_{d}) \in \mathbb{R}^{d}$,
  \[
    f_{X\mid Y=y}(x_{1},\dots,x_{d})
    \;=\;
    c_{y}\bigl(F_{1\mid y}(x_{1}),\,\dots,\,F_{d\mid y}(x_{d})\bigr)
    \;\times\;
    \prod_{i=1}^{d} f_{i\mid y}(x_{i}),
  \]
  where $f_{i\mid y}(x_{i}) = \frac{d}{dx_{i}}\,F_{i\mid y}(x_{i})$ is the $i$th marginal density.  
  In other words, the class‐conditional law factorizes exactly as “copula density $\times$ marginals.”

\item[\textbf{(2)}] \textbf{Marginal regularity (tails).} For each class $y$ and feature $i$, the marginal $X_i \mid Y=y$ has a continuous, strictly increasing $F_{i\mid y}$ with density $f_{i\mid y}$ that is bounded on compact sets and has sub-Gaussian (or at least sub-exponential) tails. Standard winsorization/monotone smoothing yields PIT variables that are well-defined and locally bi-Lipschitz on effective compact subsets; rates below are unaffected.

 \item[\textbf{(3)}] \textbf{Neural–copula approximation capacity.}
Assume the true copula density $c_y$ lies in the Sobolev class $W_2^r([0,1]^d)$ with $r>d/2$ and is bounded and Lipschitz on $[0,1]^d$. Let $\{NN_y(\cdot;\theta)\}$ denote ReLU networks whose depth and width may increase with $n$. Known approximation results for ReLU networks (e.g., Sobolev–smooth function approximation) imply that for any $\varepsilon>0$ there exists a network of size $S(\varepsilon)$ (depth $O(\log(1/\varepsilon))$ and width $O(\varepsilon^{-d/r})$ up to logarithmic factors) such that
\[
  \sup_{u\in[0,1]^d} \bigl|NN_y(u;\theta^*) - c_y(u)\bigr| \le \varepsilon .
\]
Choosing $S_n \asymp n^{d/(2r+d)}$ (up to logs) yields the sieve approximation error
\[
  \inf_{\theta:\, \text{size}\le S_n}\ 
  \sup_{u\in[0,1]^d} \bigl|NN_y(u;\theta) - c_y(u)\bigr|
  = O\!\bigl(n^{-r/(2r+d)}\bigr).
\]
Moreover, with this scaling and standard complexity controls (e.g., Rademacher/covering bounds), the estimation error matches the same order up to logarithmic factors. In practice we enforce density constraints by a positive output with normalization over $[0,1]^d$ (or a normalizing–flow parameterization); the resulting normalization error is negligible at the displayed rate.

  \item[\textbf{(4)}] \textbf{i.i.d. sampling.} 
  The training pairs $\{(X^{(i)},Y^{(i)})\}_{i=1}^{n}$ are independent and identically 
  distributed samples from the true joint law $P_{X,Y}$.  Thus we can invoke standard 
  concentration and empirical‐process arguments \citep{vandervaart2012asymptotic}.

  \item[\textbf{(5a)}] \textbf{CDF estimation.} For each $(i,y)$, the CDF estimator satisfies the DKW bound \citep{dvoretzky1956asymptotic,massart1990tight}
\[
\sup_{x\in\mathbb{R}} \bigl|\widehat F_{i\mid y}(x) - F_{i\mid y}(x)\bigr| \;=\; O_p(n_y^{-1/2}).
\]
\item[\textbf{(5b)}] \textbf{PDF estimation.} For each $(i,y)$, a kernel density estimator with standard smoothness and bandwidth choice satisfies
\[
\sup_{x\in\mathbb{R}} \bigl|\widehat f_{i\mid y}(x) - f_{i\mid y}(x)\bigr| \;=\; O_p(n_y^{-2/5}) \; \citep{silverman1986density,wasserman2006all}.
\]

Additionally, for each $(i,y)$, assume $f_{i\mid y}$ is twice continuously differentiable and bounded on its support; with a second-order kernel and bandwidth $h\asymp n_y^{-1/5}$ this yields the uniform rate $\sup_x|\widehat f_{i\mid y}(x)-f_{i\mid y}(x)|=O_p(n_y^{-2/5})$.

\noindent(In both cases $n_y$ denotes the number of samples in class $y$.)

\item[\textbf{(6)}] \textbf{Normalizer Monte Carlo accuracy.} The number of Monte Carlo samples $M$ used for $\widehat Z_y(\theta)$ obeys $M\to\infty$ and $M^{-1/2} = o\!\big(n^{-r/(2r+d)}\big)$. For instance, $M \gg n^{2r/(2r+d)}$ suffices. (Variance-reduction such as QMC/control variates can reduce constants; our rate condition uses the conservative $M^{-1/2}$ behavior.)

\item[\textbf{(7)}] \textbf{Priors and class balance.} Class priors satisfy $\min_{y}\pi_y \ge \pi_{\min} > 0$ and are consistently estimated by $\hat\pi_y = n_y/n$.

\item[\textbf{(8)}]\label{ass:regularity-mixlower} \textbf{Local mixture-density lower bound near the decision boundary.}
There exist a measurable set $S_n\subset\mathbb{R}^d$ and a constant $c_0>0$ such that $P_X(S_n)\to 1$ and
$\inf_{x\in S_n} p_X(x) \ge c_0$, where $p_X=\sum_{y=1}^K \pi_y p_y$ is the mixture density.
(This mild condition can be ensured by truncating to an “effective compact” due to the sub-Gaussian tails in Assumption~\ref{ass:regularity}(2).)

\end{enumerate}

\noindent \textbf{In the above and throughout, $n$ denotes the total number of i.i.d. training samples and $r$ is the Sobolev order of the true copula density $c_{y}\in W_{2}^{r}([0,1]^{d})$.}

\begin{theorem}[Consistency]
\label{thm:consistency}
Let 
\[
  \widehat Y_{n} \;=\;\widehat Y_{n}(\mathbf{X})
\]
be the classifier produced by the Deep Copula Classifier (DCC) when all “hat” quantities (such as \(\widehat F_{i\mid y}\), \(\widehat f_{i\mid y}\), \(\tilde c_y\), etc.) are estimated from \(n\) i.i.d.\ training samples. In particular, \(\widehat Y_{n}\) depends on \(n\) because it uses the estimated class–conditional densities
\[
  \widehat p_{y}^{(n)}(\mathbf{x})
  \;=\;
  \tilde c_{y}^{(n)}\bigl(\widehat F_{1\mid y}^{(n)}(x_{1}),\,\dots,\,\widehat F_{d\mid y}^{(n)}(x_{d})\bigr)
  \;\times\;
  \prod_{i=1}^{d}\widehat f_{i\mid y}^{(n)}(x_{i}),
\]

as outputs of Algorithm \ref{Algo1} trained on those \(n\) samples. Define the Bayes-optimal decision rule
\[
  Y^{*}(\mathbf{x})
  \;=\;
  \arg\max_{y\in\{1,\dots,K\}} \pi_y p_{y}(\mathbf{x}),
\]
where $\pi_y = P(Y=y)$ is the true class prior. Under the regularity conditions in Section~\ref{ass:regularity} (with network capacity growing appropriately as \(n\to\infty\)), we have
\[
  \lim_{n\to\infty}\;\Pbb\bigl(\widehat Y_{n}(\mathbf{X}) \neq Y\bigr)
\;=\;
\Pbb\bigl(Y^{*}(\mathbf{X}) \neq Y\bigr).
\]

In other words, the sequence of data‐driven classifiers \(\{\widehat Y_{n}\}_{n\ge1}\) is Bayes‐consistent.
\end{theorem}

\noindent
\textit{For detailed proof,} see Appendix~\ref{app:consis}.  

\begin{theorem}[Convergence Rate]
\label{thm:convergence_rate}
Let 
\[
  \widehat Y_{n} \;=\;\widehat Y_{n}(\mathbf{X})
\]
be the classifier produced by DCC when all “hat” quantities are estimated from \(n\) i.i.d.\ training samples (so in particular \(\widehat p_{y}^{(n)}(\mathbf{x})\) depends on \(n\)).  Define the excess risk
\[
  E(n)
=\Pbb\bigl(\widehat{Y}_n(\mathbf{X}) \neq Y\bigr)
-\Pbb\bigl(Y^{*}(\mathbf{X}) \neq Y\bigr),
\qquad
Y^{*}(\mathbf{x})=\arg\max_{y}\pi_y p_y(\mathbf{x}).
\]
Suppose that:
\begin{enumerate}
  \item For each class $y$, the true copula density $c_{y}\in W_{2}^{r}([0,1]^{d})$ with $r>d/2$.
  \item Assumptions \ref{ass:regularity}\,(2), (5a), (5b), (6), (7) hold. In particular,
  \[
    \sup_{x\in\mathbb{R}}\bigl|\widehat F_{i\mid y}(x)-F_{i\mid y}(x)\bigr| = O_p(n_y^{-1/2}),
    \qquad
    \sup_{x\in\mathbb{R}}\bigl|\widehat f_{i\mid y}(x)-f_{i\mid y}(x)\bigr| = O_p(n_y^{-2/5}),
  \]
  so that $\min\{1/2,\,2/5\} > r/(2r+d)$ (typically satisfied unless $d$ is very large or $r$ very small).
\end{enumerate}
Then, as \(n \to \infty\), the excess risk satisfies
\[
  E(n)
  \;=\;
  \Pbb\bigl(\widehat Y_{n}(\mathbf{X})\neq Y\bigr)
  \;-\;
  \Pbb\bigl(Y^{*}(\mathbf{X})\neq Y\bigr)
  \;=\;
  O_{p}\!\Bigl(n^{-\,\tfrac{r}{2r + d}} + n^{-\alpha}\Bigr)
  \;=\;
  O_{p}\!\Bigl(n^{-\,\tfrac{r}{2r + d}}\Bigr).
\]

neglecting any logarithmic factors.  In other words, the excess risk of the DCC classifier decays at the rate \(n^{-\,r/(2r+d)}\).
\end{theorem}

\noindent
\textit{For detailed proof,} see Appendix~\ref{app:convrate}.

\begin{proposition}[Asymptotic copula validity]\label{prop:valid}
Under Assumptions~\ref{ass:regularity}\,(2)–(3) and with $\lambda_n \to \infty$ slowly and network size $S_n \to \infty$ at the rates in Assumption~\ref{ass:regularity}\,(3), any sequence of maximizers $\hat\theta_y$ of the penalized objective satisfies
\[
\sum_{i=1}^d \int_0^1 \Big(\int \tilde c_y(u;\hat\theta_y)\,du_{-i} - 1\Big)^2 du_i \;\xrightarrow{p}\; 0.
\]
Hence $\tilde c_y(\cdot;\hat\theta_y)$ is asymptotically a valid copula density; the non-uniformity is $o_p\!\big(n^{-r/(2r+d)}\big)$ and does not affect Theorems~\ref{thm:consistency}–\ref{thm:convergence_rate}.
\end{proposition}

\noindent
\textit{For detailed proof,} see Appendix~\ref{app:proof-prop-valid}. 

\begin{corollary}[Fast excess-risk rate under a multiclass margin]\label{cor:fast-rate-mc}
Let $\eta_y(x)=P(Y=y\mid X=x)$ and define the multiclass margin
$\Delta(x)=\eta_{(1)}(x)-\eta_{(2)}(x)$, where
$\eta_{(1)}\ge \eta_{(2)}\ge\cdots$ are the ordered class probabilities.
Assume the Tsybakov margin condition $P(\Delta(X)\le t)\le C t^\kappa$ for some $\kappa>0$ and the local mixture-density condition in Assumption~\ref{ass:regularity-mixlower} below. Then, up to logarithmic factors,
\[
E(n)=O_p\!\Big(n^{-\frac{(1+\kappa)r}{2r+d}}\Big).
\]
\end{corollary}

\noindent \textit{For detailed proof,} see Appendix~\ref{app:proof-fast-rate-mc}.

In the next section, we provide the algorithms.

\section{Algorithm and Implementation}
\label{sec:algorithm}

\subsection{Training Procedure}

The DCC training process decomposes into marginal estimation and dependence learning, as formalized in Algorithm~\ref{Algo1}.

\begin{algorithm}[ht]
\caption{Theoretical Training Protocol for DCC}
\label{Algo1}
\begin{algorithmic}[1]
  \Require Labeled sample $\{(\mathbf{x}^{(i)}, y^{(i)})\}_{i=1}^n \stackrel{iid}{\sim} P_{X,Y}$
  \Ensure Estimated marginals $\{\hat{F}_{j|y}, \hat{f}_{j|y}\}$ and parameters $\{\hat{\theta}_y\}_{y=1}^K$
  \Statex
  \State \textbf{Phase 1: Consistent Marginal Estimation}
  \For{each class $y \in \{1,\dots,K\}$}
    \For{each feature $j \in \{1,\dots,d\}$}
      \State Set $\widehat F_{j|y}$ to the empirical CDF (optionally monotone-smoothed)
      \State Estimate $\widehat f_{j|y}$ by univariate KDE with bandwidth chosen by CV or a rule-of-thumb (e.g., $h\asymp n_y^{-1/5}$)
    \EndFor
  \EndFor
  \Statex
  \State \textbf{Phase 2: Neural Copula Learning}
  \For{each class $y \in \{1,\dots,K\}$}
    \State Transform: $u_j^{(i)} \leftarrow \widehat F_{j|y}\!\bigl(x_j^{(i)}\bigr)$ for all $i$ with $y^{(i)}=y$
    \State Define a positive network $NN_y(u;\theta_y)$; enforce positivity via softplus output
    \State Estimate normalizer by MC: $\widehat Z_y(\theta_y) \leftarrow \tfrac{1}{M}\sum_{m=1}^M NN_y\!\big(V^{(m)};\theta_y\big)$ with $V^{(m)}\!\sim\!\Unif([0,1]^d)$
    \State \textit{(stability)} Before taking logs, floor the normalized density: set $\tilde c_y(u;\theta_y)\leftarrow \max\!\big(\tilde c_y(u;\theta_y),\varepsilon\big)$ with $\varepsilon=10^{-12}$
    \State Maximize normalized log-likelihood:
      \[
       \hat\theta_y \;\leftarrow\; \arg\max_{\theta}\; 
\frac{1}{n_y}\sum_{i:\,y^{(i)}=y}\log \tilde c_y(u^{(i)};\theta)
\;-\; \mathcal{P}(\theta)
\quad\text{with }\mathcal{P}(\theta)\text{ from \eqref{eq:copula-penalty}.}
      \]
  \EndFor
\end{algorithmic}

\end{algorithm}

\textit{Note:} For CDFs we use the empirical CDF (optionally monotone-smoothed). For KDE bandwidth, use cross-validation or Silverman's rule \citep{silverman1986density} to balance bias--variance. \textit{Numerical stability:} when evaluating log terms, use the $\varepsilon$-flooring above to avoid $\log 0$ underflows.

\subsection{Prediction Procedure}

Classification follows the Bayes decision rule under the estimated densities (Algorithm~\ref{Algo2}).

\begin{algorithm}[H]
\caption{Theoretical Prediction Mechanism}
\label{Algo2}
\begin{algorithmic}[1]
  \Require Test point $\mathbf{x}^*$, trained DCC model
  \Ensure Predicted class $\hat{y}$
  \For{each class $y \in \{1,\dots,K\}$}
    \State $u_j^* \leftarrow \widehat F_{j|y}\!\bigl(x_j^*\bigr)$ for $j=1,\dots,d$
    \State $\tilde c_y(u^*) \leftarrow \dfrac{NN_y(u^*;\hat{\theta}_y)}{\widehat Z_y(\hat{\theta}_y)}$
    \State \textit{(stability)} \; $\tilde c_y(u^*) \leftarrow \max\!\big(\tilde c_y(u^*),\varepsilon\big)$ and $\widehat f_{j\mid y}(x_j^*) \leftarrow \max\!\big(\widehat f_{j\mid y}(x_j^*),\varepsilon\big)$ for all $j$; $\varepsilon=10^{-12}$
    \State $L_y \leftarrow \tilde c_y(u^*)\,\prod_{j=1}^d \widehat f_{j|y}(x_j^*)$
    \State $\tilde{L}_y \leftarrow L_y \cdot \hat{\pi}_y$ with $\hat{\pi}_y = n_y/n$
    \State \textit{(optional calibration)} If using post-hoc calibration, pass the decision score through the Platt sigmoid learned on the calibration split to obtain probabilities; AUC is unchanged by this monotone transform.
  \EndFor
  \State \textit{(tie-break)} If multiple $y$ achieve the maximum of $\tilde L_y$, return the smallest index
  \State \Return $\hat{y} = \arg\max_{y} \tilde{L}_y$
\end{algorithmic}

\end{algorithm}

\noindent\textit{Complexity remark.} Prediction evaluates $\widehat F_{j\mid y}$ for each $j=1,\dots,d$ and each class $y$, and one copula network per class. The overall per-point cost is $O(Kd)$ CDF evaluations plus $K$ network passes; vectorization/caching across $y$ reduces wall-time in practice.

\subsection{Neural Network Architecture}
\label{sec:nnarch}

The copula networks $NN_y$ require architectural constraints to ensure theoretical validity:

\begin{itemize}
    \item \textbf{Input/Output:} $NN_y:[0,1]^d \to (0,\infty)$; positivity via softplus output. The copula density used in training/prediction is the normalized $\tilde c_y(u;\theta)=NN_y(u;\theta)/\widehat Z_y(\theta)$ with $\widehat Z_y$ estimated by MC as above.
    \item \textbf{Depth/Width (theory scale):} Depth $L = O(\log n)$ and width $W = O\!\big(n^{d/(2r+d)}\big)$ suffice for $r$-smooth copulas (rates up to logarithmic factors).
    \item \textbf{Activation:} ReLU in hidden layers; softplus output for positivity.
    \item \textbf{Regularization:} Spectral normalization or weight decay to control complexity (consistent with Rademacher/covering bounds used in the proofs).
\end{itemize}

We apply spectral normalization or weight decay to keep the estimator in a bounded Sobolev/Lipschitz ball, which is used in the sup-norm approximation and estimation arguments.
\citep{miyato2018spectral,bartlett2017spectrally}

\textit{Note:} While theoretically achievable through Barron space approximation \citep{barron1993universal}, practical implementations enforce the integral constraint via normalization techniques like softmax over a grid or Monte Carlo integration.

In the next section, we report the results from our experiments.


\section{Experiments}\label{sec:experiments}

We present results from two testbeds: first, a controlled, synthetic setting with a known ground-truth dependence, and second, a real-world clinical benchmark. To prevent tuning bias, all methods use the same train, calibration, and test splits. Calibration is performed only on the held-out calibration split, while the test set is not used until the final evaluation.

\subsection{Experiment 1: Synthetic bivariate classes with opposite correlations}\label{sec:exp-toy}

\paragraph{Data.}
We generate two binary classes as
$\mathcal{N}(\mathbf{0}, \Sigma_{+\rho})$ vs.\ $\mathcal{N}(\mathbf{0}, \Sigma_{-\rho})$,
with $\Sigma_{\rho}=\begin{psmallmatrix}1 & \rho\\ \rho & 1\end{psmallmatrix}$ and $\rho=0.995$.
We sample $2{,}000$ points per class (total $n=4{,}000$). A fixed split is used throughout:
$70\%$ train and $30\%$ test (stratified); from the training set, a $15\%$ \emph{calibration} subset
is carved out (stratified) and used only for scalar probability calibration.

\paragraph{Model.}
DCC uses a positive MLP with spectral normalization; the copula density is normalized by a
deterministic grid expectation over $[0,1]^2$ (grid resolution $256\times256$).
Depth and width follow the theory-guided recipe in the code
($L=\lceil \log_2 n_y\rceil$, width $\propto n_y^{d/(2r+d)}$ with $r{=}2$).
Training uses Adam ($\text{lr}=10^{-3}$, $500$ epochs, batch size $256$) with a mild grid penalty
encouraging uniform marginals. \textbf{Decision rule:} we use $\tau=1$, i.e., the copula term
\emph{plus} the product of univariate marginals and class priors.\footnote{Platt (logistic)
calibration is fitted only on the held-out calibration split and then applied to test scores.}
We ablate the marginal strategy: \texttt{oracle\_normal} (true standard-normal marginals),
\texttt{pooled} KDE, and \texttt{per\_class} KDE, all with the same paper-faithful bandwidth
$h \propto n^{-0.51}$ (scale $=10$).

\paragraph{Figures.}
The data and decision regions are shown in Fig.~\ref{fig:toy-grid}; the boundaries in the three
panels use the calibrated $0.5$ threshold. Oracle and pooled marginals recover the expected
quadrants with sharp transitions; the per-class KDE shows the anticipated roughness due to fewer
samples per class.

\begin{figure}[ht]
  \centering
  \includegraphics[width=.48\linewidth]{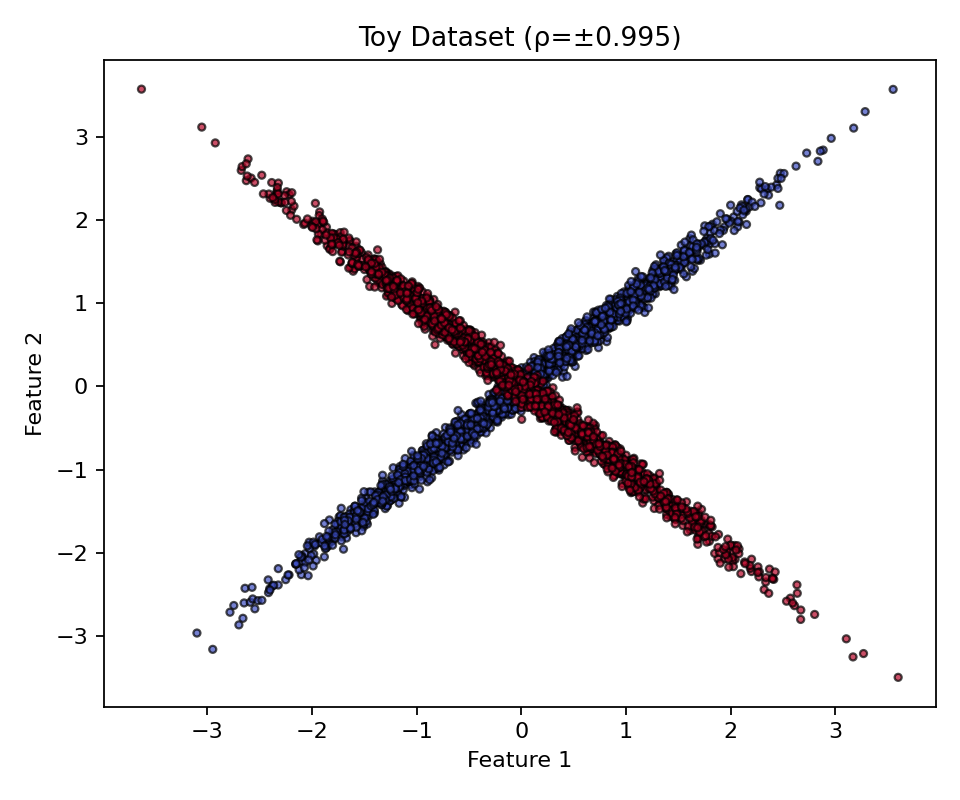}
  \includegraphics[width=.48\linewidth]{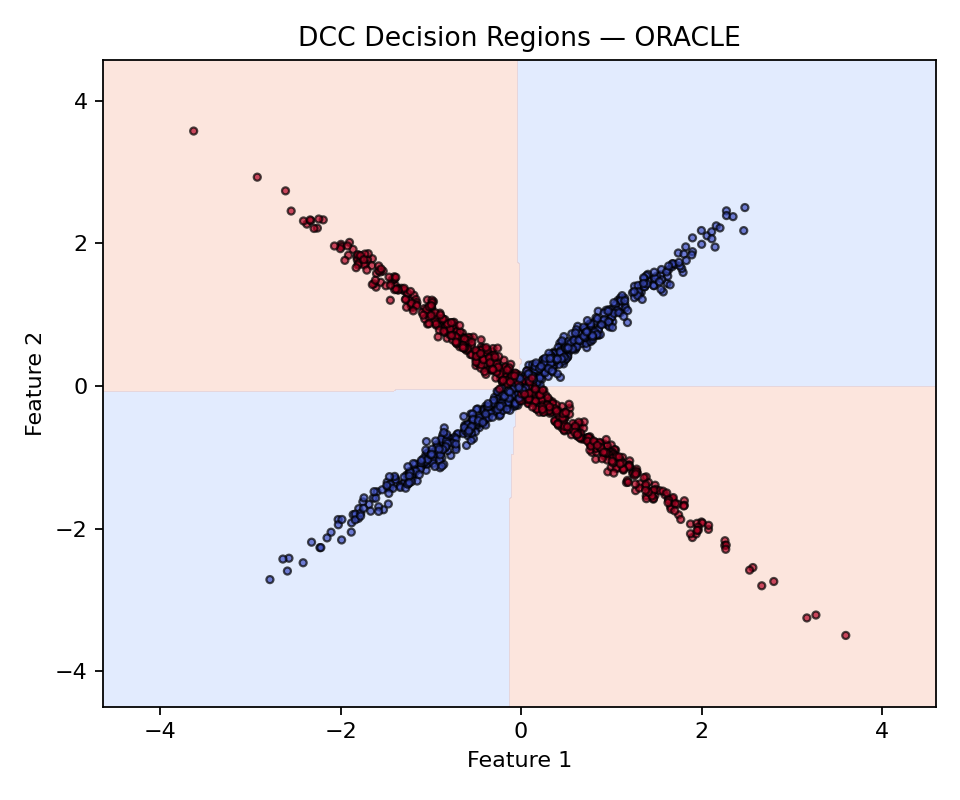}\\
  \vspace{0.25em}
  \includegraphics[width=.48\linewidth]{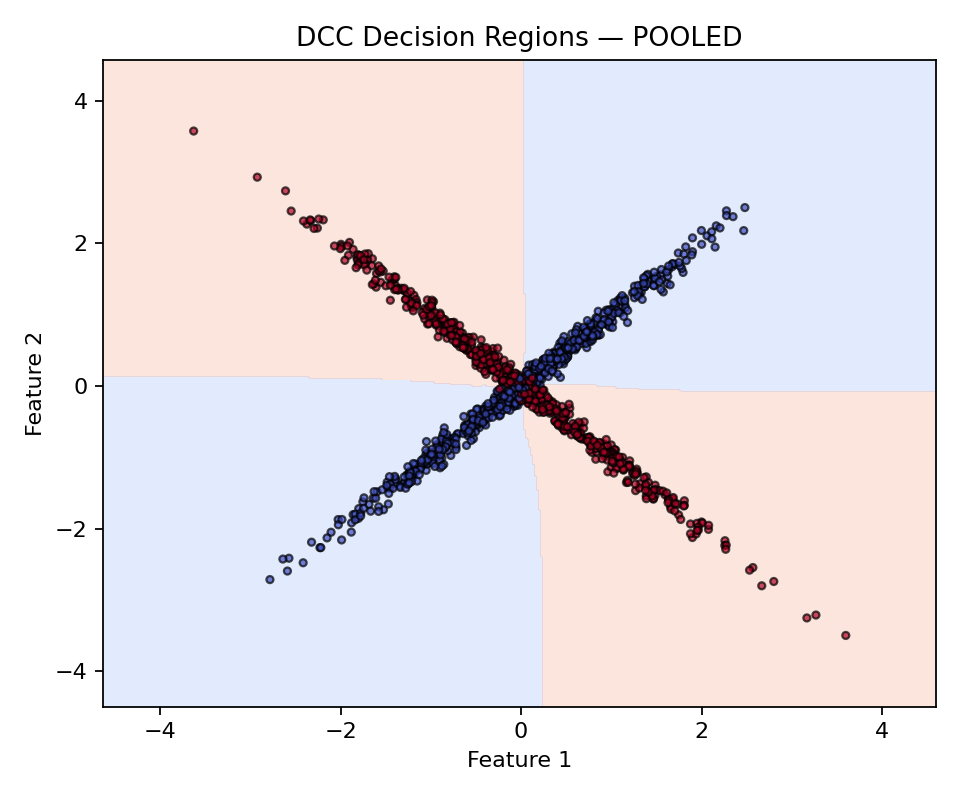}
  \includegraphics[width=.48\linewidth]{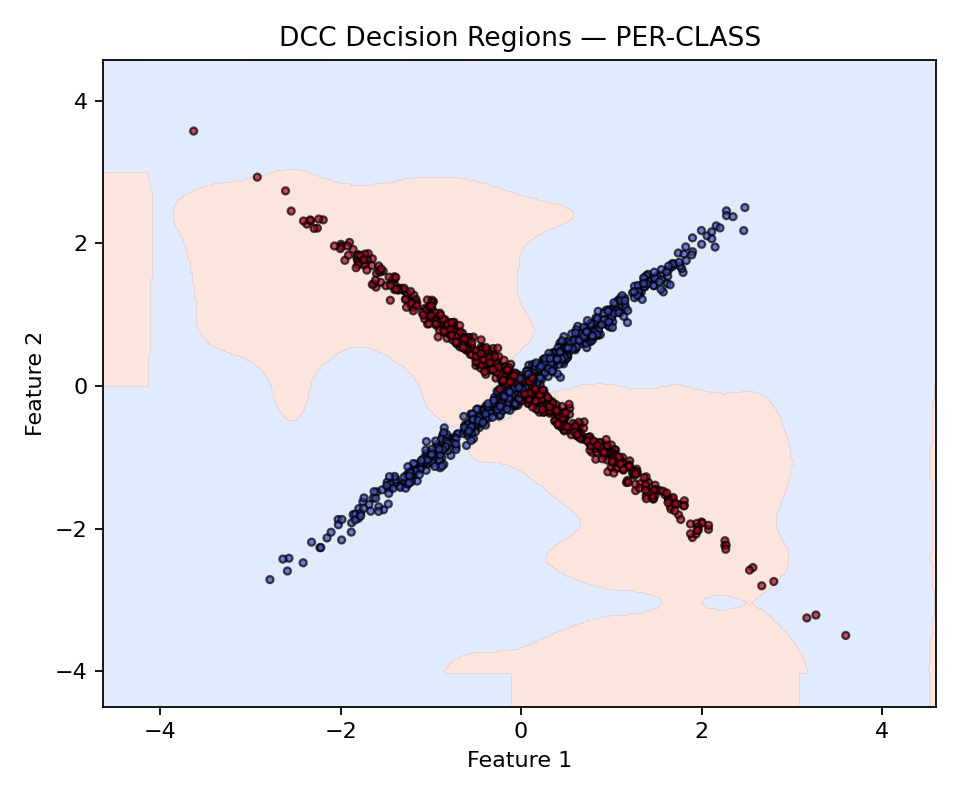}
  \caption{\textbf{Experiment 1 (synthetic).} Top-left: data ($\rho=\pm0.995$).
  Top-right/bottom-left/bottom-right: DCC decision regions with
  \texttt{oracle\_normal}, \texttt{pooled} KDE, and \texttt{per\_class} KDE marginals.}
  \label{fig:toy-grid}
\end{figure}

\paragraph{Results.}
Table~\ref{tab:toy} reports test Accuracy, ROC-AUC, and PR-AUC for one fixed seeded run (identical
hyperparameters across ablations). Oracle and pooled match the Monte-Carlo Bayes ceiling
($\approx 0.97$ accuracy for $|\rho|=0.995$) while the per-class KDE underperforms, as expected,
because each marginal is fit on half as many samples.

\begin{table}[ht]
\centering
\caption{\textbf{Experiment 1 (synthetic):} DCC ablations (calibrated). Higher is better.}
\label{tab:toy}
\begin{tabular}{lccc}
\toprule
\textbf{marginal\_mode} & \textbf{Accuracy} & \textbf{ROC-AUC} & \textbf{PR-AUC} \\
\midrule
\texttt{oracle\_normal} & 0.970000 & 0.997442 & 0.997527 \\
\texttt{pooled}         & 0.970833 & 0.997625 & 0.997674 \\
\texttt{per\_class}     & 0.873333 & 0.956844 & 0.965534 \\
\bottomrule
\end{tabular}
\end{table}

\subsection{Dataset Description (PIMA)}\label{sec:dataset-pima}
We additionally evaluate our method on a real-world dataset, namely the \textit{Pima Indians Diabetes Database}, a widely used clinical benchmark originating from the National Institute of Diabetes and Digestive and Kidney Diseases (NIDDK) and described by \citet{smith1988using}. The dataset comprises 768 adult female patients (age $\geq$ 21) of Pima Indian heritage, with \textbf{eight} diagnostic predictors and a binary diabetes label \texttt{Outcome}. A curated copy of this dataset was obtained from Kaggle.

\textbf{Target Variable:} The dataset includes a binary label \texttt{Outcome} (0 = no diabetes, 1 = diabetes). We use the label exactly as it appears in the repository. We did not relabel or combine any classes. Our study approaches diabetes risk prediction as a binary classification problem, matching the repository's label. This approach allows us to focus on evaluating risk factors and classification performance.

\textbf{Features:} The eight predictors are: \emph{Pregnancies}, \emph{Glucose}, \emph{BloodPressure}, \emph{SkinThickness}, \emph{Insulin}, \emph{BMI}, \emph{DiabetesPedigreeFunction}, and \emph{Age}. For completeness, Table~\ref{tab:pima_features} lists the variables and brief descriptions.

\begin{table}[ht]
\centering
\caption{\textbf{Summary of features in the PIMA dataset.} Brief one-line descriptions of each diagnostic predictor.}
\label{tab:pima_features}
\begin{tabularx}{\textwidth}{@{} l X @{}}
\toprule
\textbf{Feature} & \textbf{Description} \\
\midrule
Pregnancies & Number of pregnancies. \\
Glucose & 2-hour plasma glucose concentration (oral glucose tolerance test). \\
BloodPressure & Diastolic blood pressure (mm Hg). \\
SkinThickness & Triceps skinfold thickness (mm). \\
Insulin & 2-hour serum insulin (mu U/ml). \\
BMI & Body mass index $\left(\frac{\text{kg}}{\text{m}^2}\right)$. \\
DiabetesPedigreeFunction & Proxy for hereditary risk (family history signal). \\
Age & Age in years. \\
\bottomrule
\end{tabularx}
\end{table}

\textbf{Data characteristics:} The working CSV (768 $\times$ 9) contains no NA entries. As is standard for this dataset, several physiological fields may contain literal zeros (\emph{Glucose}, \emph{BloodPressure}, \emph{SkinThickness}, \emph{Insulin}, \emph{BMI}); these are present in the released files and are not explicit NA placeholders.

\subsection{Experiment 2: Real-world classification on PIMA}\label{sec:exp-pima}

\paragraph{Preprocessing.}
We treat zeros in \{\emph{Glucose}, \emph{BloodPressure}, \emph{SkinThickness}, \emph{Insulin}, \emph{BMI}\} as missing (set to NaN). To avoid leakage, we first split the data, then learn \emph{per-class} medians and per-class winsor bounds at $q=0.005$ on the \textit{fit} subset only, and apply those statistics to the calibration and test splits. No relabeling or class collapsing is performed.

\paragraph{Splits and calibration.}
A single fixed split is used for all methods: $70\%$ train/$30\%$ test (both stratified). From train, a $15\%$ stratified \emph{calibration} subset is carved out (held out from model fitting). Each model is fit on the fit subset only, then Platt (logistic) calibration is learned on the calibration subset and applied to the test scores. \emph{No hyperparameter tuning} is performed.

\paragraph{Models.}
\textbf{DCC} uses the paper-faithful $\tau{=}1$ rule,
$p(y\!\mid\!x)\propto \pi_y\, c_y(U)\,\prod_j f_{j\mid y}(x_j)$ with \emph{per-class} KDE marginals (bandwidth $h \propto n^{-0.51}$, scale $=10$). The positive MLP employs spectral normalization; depth $L=\lceil\log_2 n_y\rceil$ and width $\propto n_y^{d/(2r+d)}$ with $r{=}12$. The copula normalizer is a Sobol QMC expectation with $65{,}536$ points. Training uses Adam for $700$ epochs, batch size $128$, learning rate $8\times 10^{-4}$, and a mild histogram penalty ($\lambda=0.1$) encouraging uniform marginals on the QMC grid.
\textbf{Baselines} (scikit-learn defaults): Logistic Regression (with standardization), SVM with RBF kernel (with standardization; we calibrate its \texttt{decision\_function}), Random Forest (default, 100 trees), and Gaussian Naïve Bayes. All baselines are fit on the same fit split and Platt-calibrated on the calibration split.

\paragraph{Figures.}
Test-set ROC and PR curves (top row) and reliability diagrams (bottom) are shown in Fig.~\ref{fig:pima-curves}. Curves use the \emph{calibrated} probabilities. The reliability diagram reports Expected Calibration Error (ECE) with 10 uniform bins.

\begin{figure}[ht]
  \centering
  \includegraphics[width=.48\linewidth]{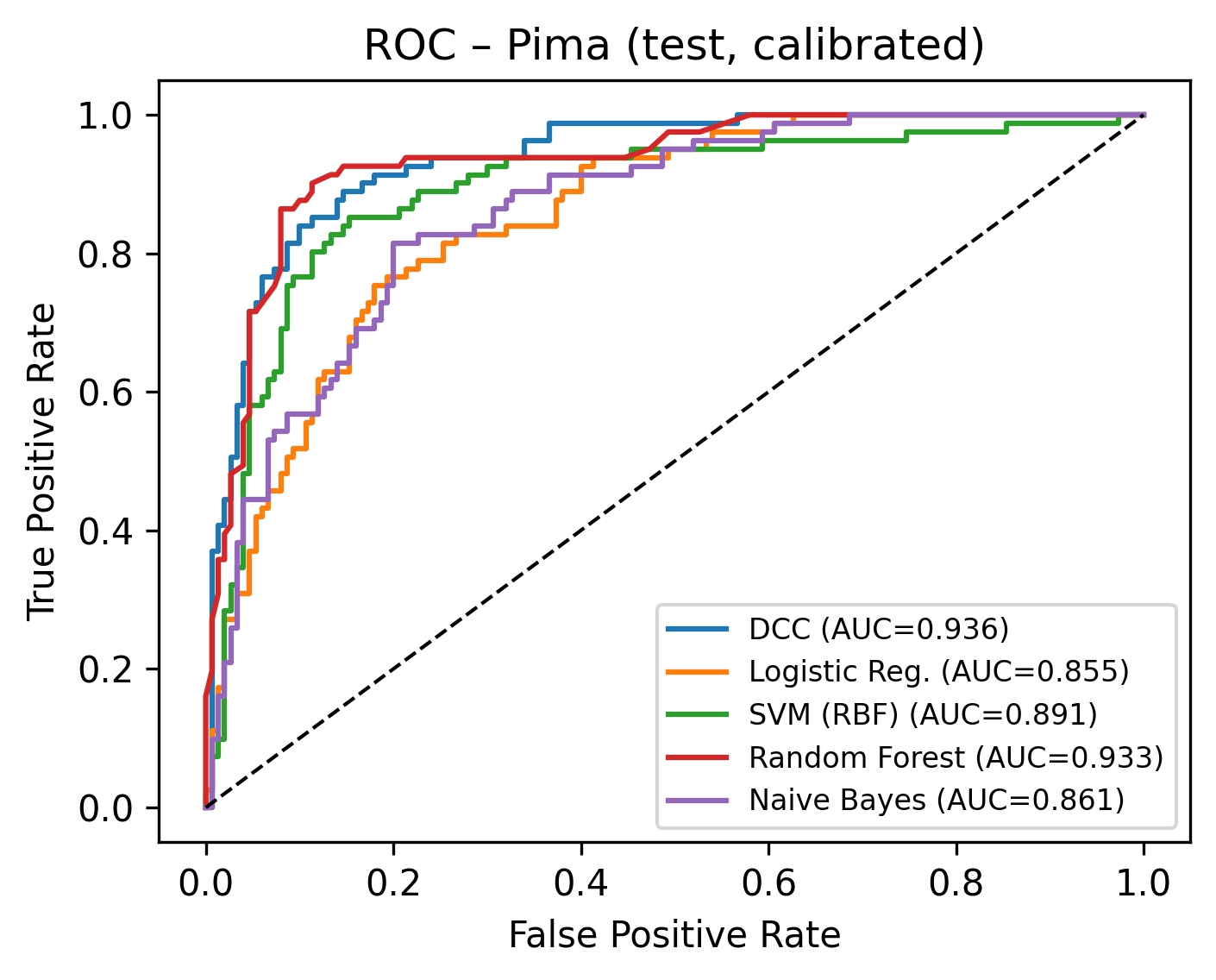}\hfill
  \includegraphics[width=.48\linewidth]{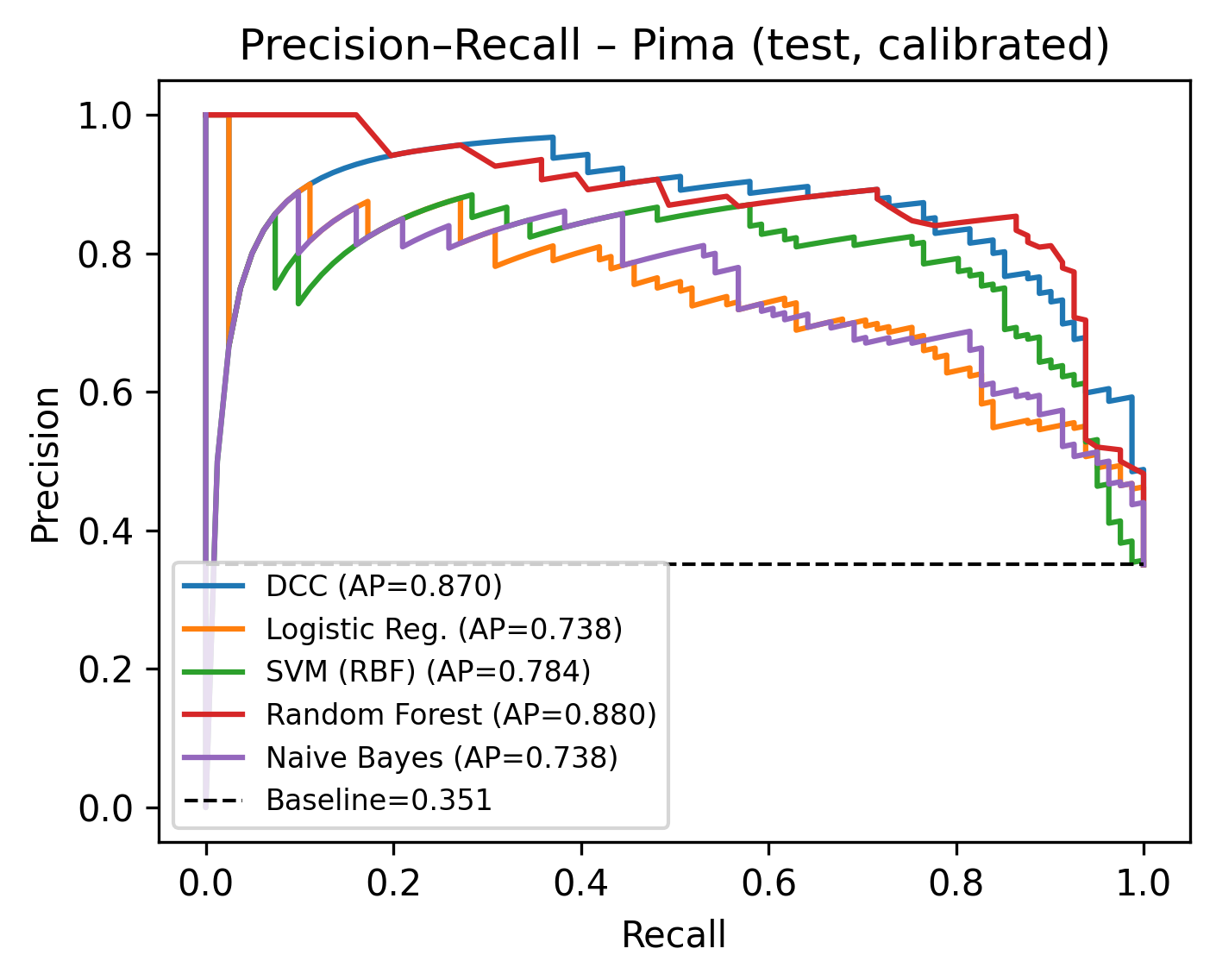}\\[0.35em]
  \includegraphics[width=.60\linewidth]{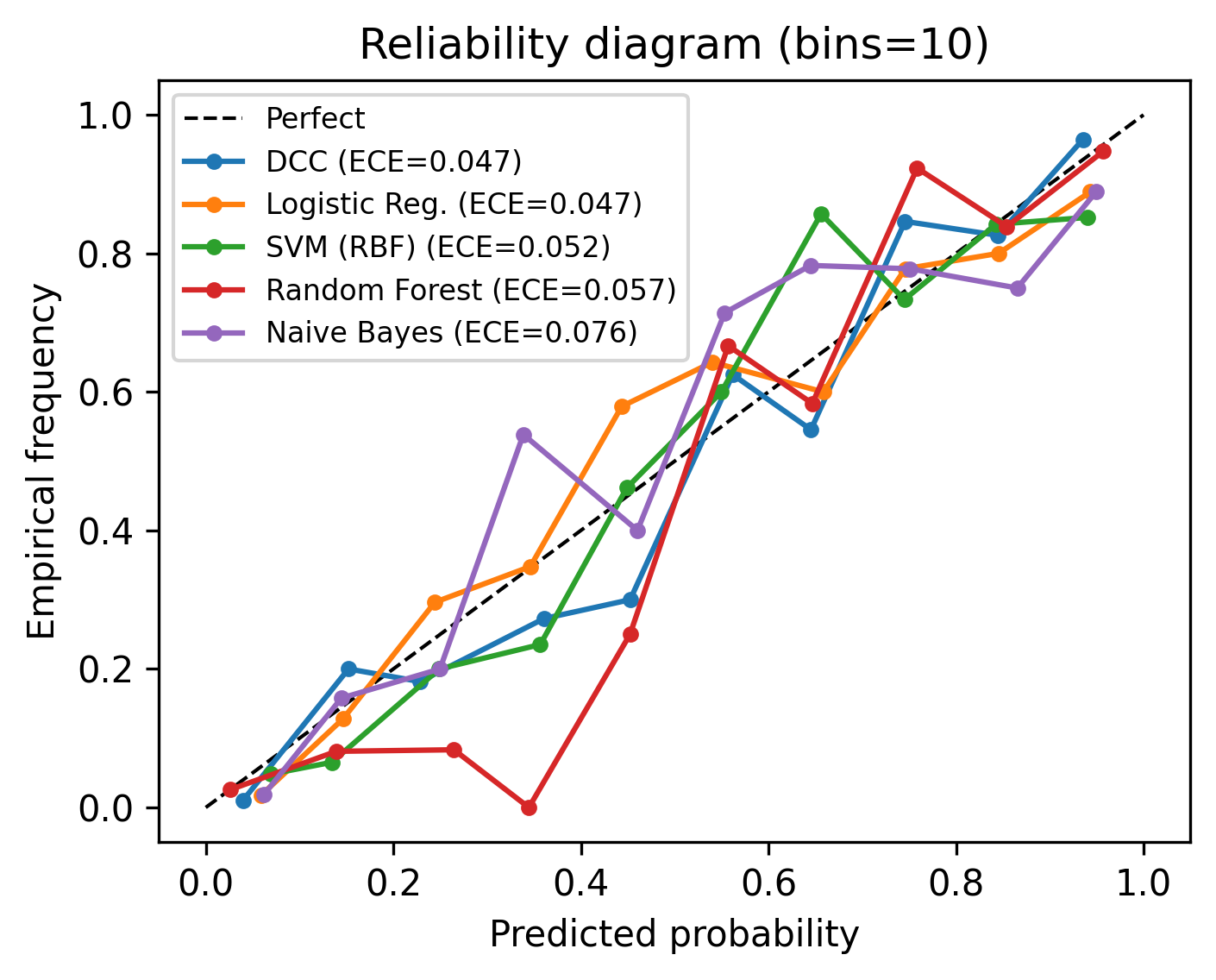}
  \caption{\textbf{PIMA (test, calibrated).} Top: ROC (left) and Precision--Recall (right).
  Bottom: reliability diagram (ECE in legend).}
  \label{fig:pima-curves}
\end{figure}

\paragraph{Calibration (ECE).}
The reliability diagram in Fig.~\ref{fig:pima-curves} is computed on the \emph{calibrated} probabilities using 10 uniform bins. DCC and Logistic Regression achieve the lowest Expected Calibration Error (ECE $=0.047$ each), followed by SVM ($0.052$), Random Forest ($0.057$), and Naïve Bayes ($0.076$). Thus, under the same Platt calibration, DCC’s probability estimates are as well-calibrated as the best classical baseline and better than the others.

\paragraph{Results.}
Table~\ref{tab:pima-results} summarizes Accuracy, ROC-AUC, and PR-AUC on the held-out test set (single fixed split). DCC attains the highest ROC-AUC and ties Logistic Regression for the best calibration (ECE $\approx 0.047$); Random Forest leads slightly on Accuracy and PR-AUC.

\begin{table}[ht]
\centering
\caption{\textbf{Experiment 2 (PIMA):} Test metrics with Platt calibration (no tuning). Higher is better.}
\label{tab:pima-results}
\begin{tabular}{lccc}
\toprule
\textbf{Model} & \textbf{Accuracy} & \textbf{ROC-AUC} & \textbf{PR-AUC} \\
\midrule
DCC ($\tau{=}1$, calibrated)      & 0.879 & \textbf{0.936} & 0.870 \\
Logistic Regression (calibrated)  & 0.779 & 0.855          & 0.738 \\
SVM (RBF, calibrated)             & 0.848 & 0.891          & 0.784 \\
Random Forest (calibrated)        & \textbf{0.892} & 0.933 & \textbf{0.880} \\
Na\"ive Bayes (calibrated)        & 0.792 & 0.861          & 0.738 \\
\bottomrule
\end{tabular}
\end{table}

\paragraph{Takeaways.}
On this small tabular benchmark, RF is strongest on Accuracy and PR-AUC; \textbf{DCC} is competitive overall, achieves the best ROC-AUC, and shows top-tier calibration (ECE $\approx$ 0.047, tied with Logistic Regression), while substantially outperforming the classical linear/probabilistic baselines on the discriminative metrics.

\paragraph{Reproducibility.}

\section{Conclusion and Future Work}\label{sec:con}

Our theoretical analysis shows that the Deep Copula Classifier (DCC) is Bayes-consistent and models feature dependencies using neural copula parameterization. By learning marginals and dependencies separately, DCC offers an interpretable framework. Practitioners can examine individual feature behaviors using $\hat{F}_{j|y}$ and interaction patterns with $NN_y$. The $O(n^{-r/(2r+d)})$ convergence rate for Sobolev-class copulas highlights the method's ability to adapt to the complexity of the data.

We supported our theory with two controlled studies that followed a consistent and fair protocol. All methods used identical fixed train, calibration, and test splits. No hyperparameter tuning was performed (defaults only). Models were fit on the training portion, a single Platt calibrator was learned on the calibration split, and the test set was held out for a one-shot evaluation. In Experiment 1, using correlated Gaussians with $|\rho|=0.995$, DCC learned decision regions that match the Bayes boundary. With the same hyperparameters across marginal strategies, the \texttt{oracle\_normal} and \texttt{pooled} settings reached near-ceiling performance (Accuracy $\approx 0.971$, ROC–AUC $\approx 0.998$, PR–AUC $\approx 0.998$), while \texttt{per\_class} underperformed due to finite-sample marginal estimation (Accuracy $=0.873$, ROC–AUC $=0.957$, PR–AUC $=0.966$). These results support the main claim that modeling dependencies directly helps when correlation drives class separation. On the Pima Indians Diabetes dataset, the calibrated DCC with $\tau=1$ achieved Accuracy $=0.879$, ROC–AUC $=0.936$, and PR–AUC $=0.870$. Among the baselines, Random Forest performed best on Accuracy and PR–AUC (Accuracy $=0.892$, ROC–AUC $=0.933$, PR–AUC $=0.880$), while DCC attained the highest ROC–AUC. Logistic Regression, SVM (RBF), and Gaussian Naïve Bayes had lower ROC–AUC scores than DCC ($0.855$, $0.891$, and $0.861$, respectively). Reliability curves show that DCC and Logistic Regression have the lowest and essentially tied Expected Calibration Error on the test set. Tree ensembles remain a reliable choice for small tabular problems. DCC works well when cross-feature dependence is important, when calibrated probabilities and an interpretable structure are needed, or when a compact generative classifier with clear likelihoods is preferred. Our results and theory indicate that DCC is a solid and fair baseline for tasks that require accounting for feature dependence.

In our experiments, we used $d\leq 8$ and moderate sample sizes. Scaling to higher $d$ is still difficult because of the challenges in estimating copula densities on $[0,1]^d$. We also set $\tau=1$ and used a basic per-class KDE. Exploring more complex marginal models and structured priors could lead to better results.

There are several ways to extend this work. The curse of dimensionality in copula estimation on $[0,1]^d$ poses significant challenges. To address this, one promising direction involves vine copula factorizations (\cite{nagler2017nonparametric}), which break down high-dimensional dependencies into tree-structured graphs that are easier to estimate. Additionally, low-rank tensor decompositions of the weight matrices within $NN_y$ may reduce parameter complexity while preserving expressive power. Another line of work would analyze the tradeoff between sparsity and sample size in regimes where $d \gg n$, providing theoretical insights into when DCC remains statistically reliable. Given its generative nature, DCC naturally lends itself to semi-supervised extensions under the cluster assumption (\cite{chapelle2009semi}). This opens the door to convergence analyses for entropy-regularized objectives that leverage both labeled and unlabeled data. In particular, DCC may benefit when marginal distributions $P(X)$ carry information about class boundaries, leading to provable gains in label efficiency. Furthermore, a rigorous comparison between DCC and discriminative semi-supervised methods could clarify the advantages of explicitly modeling dependencies. In streaming data settings, DCC can be extended to incorporate online learning mechanisms. One direction is to establish regret bounds for sequential updates to the copula density via mirror descent (\cite{jadbabaie2015online}). Beyond regret analysis, theoretical guarantees will be needed to understand how DCC behaves under class-incremental scenarios, especially regarding catastrophic forgetting. Additionally, DCC provides a unique opportunity to model and detect concept drift through shifts in learned dependence structures over time.

Although our focus has been theoretical, empirical observations raise questions that demand mathematical exploration. For example, DCC may exhibit phase transitions in classification performance as the strength of feature dependencies varies, phenomena that could be characterized through theoretical thresholds. Another promising direction is to derive information-theoretic lower bounds for copula-based classification, which would help identify when dependency modeling provides fundamental advantages over independence-based models.

The DCC framework opens new questions about balancing parametric and nonparametric components in hybrid models. Maintaining explicit probability models rather than black-box representations enables a precise characterization of when and how feature dependencies impact classification, a crucial step toward explainable AI for high-stakes decisions.


\section*{Code Availability}
Code will be released publicly upon acceptance.

\section*{Data Availability}
The PIMA Indians Diabetes dataset used in our Experiment 2 originates from NIDDK and \citet{smith1988using}; a curated copy used for this work was accessed via Kaggle.

\section*{Declarations}
\subsection*{Competing Interests}
The authors declare no competing interests.



\appendix
\section{Detailed Proofs}
\label{sec:appendix}

\subsection{A generative plug-in regret bound}\label{app:gen-bound}
\begin{lemma}[Generative regret bound for plug-in MAP]
Let $a_y(x)=\pi_y p_y(x)$ and $\hat a_y(x)=\hat\pi_y\,\widehat p_y(x)$ be nonnegative functions with $\sum_y a_y(x)=p_X(x)$ and $\sum_y \hat a_y(x)=\hat p_X(x)$ integrable on $\mathbb{R}^d$. 
Let $Y^*(x)=\arg\max_y a_y(x)$ and $\widehat Y(x)=\arg\max_y \hat a_y(x)$. Then
\[
R(\widehat Y)-R(Y^*) \;\le\; \int \sum_{y=1}^K \big|a_y(x)-\hat a_y(x)\big|\,dx.
\]
\end{lemma}

\begin{proof}
Write $R(g)=1-\int a_{g(x)}(x)\,dx$ and $R(Y^*)=1-\int \max_y a_y(x)\,dx$. Thus
\[
R(\widehat Y)-R(Y^*) \;=\; \int\!\big[\max_y a_y(x)-a_{\widehat Y(x)}(x)\big]\,dx.
\]
For each $x$, add and subtract $\max_y \hat a_y(x)$ and use that $\widehat Y(x)$ maximizes $\hat a$:
\[
\max_y a_y - a_{\widehat Y} 
\;\le\; \big(\max_y a_y - \max_y \hat a_y\big) + \big(\hat a_{\widehat Y}-a_{\widehat Y}\big)
\;\le\; \max_y|a_y-\hat a_y| + |\hat a_{\widehat Y}-a_{\widehat Y}|
\;\le\; \sum_{y=1}^K |a_y-\hat a_y|.
\]
Integrate over $x$ to conclude.
\end{proof}

\subsection{Proof of Theorem \ref{thm:consistency} (Consistency)} \label{app:consis}

We prove that, under the regularity conditions in Section~\ref{ass:regularity}, the Deep Copula Classifier (DCC) converges to the Bayes‐optimal classifier as \(n \to \infty\).  Recall that every “hat” quantity (\(\widehat F_{i\mid y},\,\widehat f_{i\mid y},\,\tilde c_y,\,\widehat p_{y},\,\widehat Y\), etc.) depends on the total sample size \(n\) because it is estimated from \(n\) i.i.d.\ training pairs.

\paragraph{Step 1: Marginal Convergence}  
By Assumption \ref{ass:regularity}(5),
\[
  \sup_{x \in \mathbb{R}}
  \bigl|\widehat F_{i\mid y}(x) - F_{i\mid y}(x)\bigr|
  \;=\;
  O_{p}\bigl(n^{-\alpha}\bigr),
  \qquad
  \alpha > \frac{r}{2r + d}.
\]
Define, for each training point \(X_i\),
\[
  u_i = F_{i\mid y}(X_i),
  \qquad
  \widehat u_i = \widehat F_{i\mid y}(X_i).
\]
Then directly,
\[
  \lvert \widehat u_i - u_i \rvert
  \;=\;
  \bigl|\widehat F_{i\mid y}(X_i) - F_{i\mid y}(X_i)\bigr|
  \;\le\;
  \sup_{x}\bigl|\widehat F_{i\mid y}(x) - F_{i\mid y}(x)\bigr|
  \;=\; O_p(n^{-\alpha}).
\]
Therefore the vector of pseudo‐observations
\(\widehat{\mathbf{u}} = (\widehat u_1,\dots,\widehat u_d)\) satisfies
\[
  \bigl\|\widehat{\mathbf{u}} - \mathbf{u}\bigr\|_{\infty}
  \;=\;
  O_{p}\bigl(n^{-\alpha}\bigr).
\]
(Assumption \ref{ass:regularity}(2) gives local bi-Lipschitz behavior away from boundaries; standard Winsorization/monotone smoothing near the boundary keeps rates unchanged.)

\paragraph{Step 2: Neural Copula Approximation}  
Assumption \ref{ass:regularity}(3) guarantees that our chosen class of neural networks can approximate any Sobolev‐smooth copula density.
In particular, since the true class‐conditional copula \(c_{y} \in W_{2}^{r}([0,1]^{d})\) with \(r > d/2\), \cite{schmidt2020nonparametric} shows there exists a ReLU network of depth \(O(\log n)\) and width \(O\bigl(n^{d/(2r+d)}\bigr)\) whose parameters \(\theta_{y}^{*}\) satisfy
\[
  \sup_{u \in [0,1]^{d}}
  \bigl|\,NN_{y}(u;\theta_{y}^{*}) - c_{y}(u)\bigr|
  \;=\;
  O\!\bigl(n^{-\,r/(2r+d)}\bigr).
\]
Let \(\tilde c_{y}(u) = \dfrac{NN_{y}(u;\widehat\theta_{y})}{\widehat Z_y(\widehat\theta_{y})}\) denote the normalized network after training on \(n_y\) samples, with \(\widehat Z_y(\widehat\theta_{y}) = \frac{1}{M}\sum_{m=1}^M NN_y(V^{(m)};\widehat\theta_{y})\).

By standard Rademacher‐complexity‐based MLE arguments (\cite{koltchinskii2001rademacher}), one shows
\[
  \mathbb{E}_{u \sim \mathrm{Unif}[0,1]^{d}}
  \bigl[\bigl|\,NN_{y}(u;\widehat\theta_{y}) - c_{y}(u)\bigr|^{2}\bigr]
  \;=\;
  O_{p}\!\Bigl(n^{-\,\tfrac{2r}{2r + d}} \;+\; n^{-1}\Bigr).
\]

Passing to the normalized estimator $\tilde c_y(u) = NN_y(u;\widehat\theta_y) / \widehat Z_y(\widehat\theta_y)$, write
\[
\tilde c_y(u) - c_y(u) = \frac{NN_y(u;\widehat\theta_y) - c_y(u)}{\widehat Z_y(\widehat\theta_y)} \;+\; c_y(u)\!\left(\frac{1}{\widehat Z_y(\widehat\theta_y)} - 1\right).
\]
Hence
\[
\|\tilde c_y - c_y\|_\infty \;\le\; \|NN_y(\cdot;\widehat\theta_y) - c_y\|_\infty \cdot \bigl|1/\widehat Z_y\bigr| \;+\; \|c_y\|_\infty \cdot \bigl|1/\widehat Z_y - 1\bigr|.
\]
By the approximation and estimation bounds above, $\|NN_y - c_y\|_\infty = O_p(n^{-r/(2r+d)})$ (up to logs). Moreover, $\widehat Z_y(\widehat\theta_y)=\int NN_y + O_p(M^{-1/2})$ and $\int NN_y = 1 + O_p(n^{-r/(2r+d)})$ because $NN_y$ uniformly approximates $c_y$ and $\int c_y = 1$. Thus $|\widehat Z_y - 1| = O_p(n^{-r/(2r+d)}) + O_p(M^{-1/2})$, so choosing $M$ as in Assumption (6) yields
\[
\|\tilde c_y - c_y\|_\infty \;=\; O_p\!\bigl(n^{-r/(2r+d)}\bigr).
\]

\paragraph{Step 3: Decision Rule Stability}  
For any test input \(x=(x_{1},\dots,x_{d})\), define
\[
  p_{y}(\mathbf{x})
  \;=\;
  c_{y}\bigl(F_{1\mid y}(x_{1}),\,\dots,\,F_{d\mid y}(x_{d})\bigr)
  \;\times\;\prod_{i=1}^{d}f_{i\mid y}(x_{i}),
\]
and its estimate
\[
  \widehat p_{y}(\mathbf{x})
  \;=\;
  \tilde c_{y}\bigl(\widehat{\mathbf{u}}\bigr)
  \;\times\;\prod_{i=1}^{d}\widehat f_{i\mid y}(x_{i}),
  \quad
  \text{where } \widehat{\mathbf{u}} = \bigl(\widehat u_{1},\dots,\widehat u_{d}\bigr).
\]
We bound 
\begin{align*}
  \bigl|\widehat p_{y}(\mathbf{x}) - p_{y}(\mathbf{x})\bigr|
  &= 
  \Bigl|\,
    \tilde c_{y}(\widehat{\mathbf{u}})\,\prod_{i=1}^{d}\widehat f_{i\mid y}(x_{i})
    \;-\; 
    c_{y}(\mathbf{u})\,\prod_{i=1}^{d}f_{i\mid y}(x_{i})
  \Bigr| \\[6pt]
  &\le 
  \underbrace{\bigl|\tilde c_y(\widehat{\mathbf{u}}) - c_{y}(\mathbf{u})\bigr|\;\prod_{i=1}^{d}\widehat f_{i\mid y}(x_{i})}
    _{\text{(A) Copula‐error}}
  \;+\;
  \underbrace{\bigl|c_{y}(\mathbf{u})\bigr|\;\bigl|\prod_{i=1}^{d}\widehat f_{i\mid y}(x_{i})
    - \prod_{i=1}^{d}f_{i\mid y}(x_{i})\bigr|}
    _{\text{(B) Marginal‐error}}.
\end{align*}

\noindent \textbf{(A) Copula‐error.}  Observe
\[
  \bigl|\tilde c_{y}(\widehat{\mathbf{u}}) - c_{y}(\mathbf{u})\bigr|
  \;\le\;
  \bigl|\tilde c_{y}(\widehat{\mathbf{u}}) - c_{y}(\widehat{\mathbf{u}})\bigr|
  \;+\;
  \bigl|c_{y}(\widehat{\mathbf{u}}) - c_{y}(\mathbf{u})\bigr|.
\]
1. From Step 2,
\(\sup_{u}\lvert\tilde c_{y}(u) - c_{y}(u)\rvert = O_{p}\!\bigl(n^{-\,r/(2r+d)}\bigr).\)  
2. By Assumption (3), $c_y$ is Lipschitz continuous on $[0,1]^d$. Let $L_c$ denote its Lipschitz constant. Then
\[
  \bigl|\,c_{y}(\widehat{\mathbf{u}}) - c_{y}(\mathbf{u})\bigr|
  \;\le\;
  L_{c}\,\bigl\|\widehat{\mathbf{u}} - \mathbf{u}\bigr\|_{\infty}
  \;=\;
  O_{p}\bigl(n^{-\alpha}\bigr)
  \quad(\text{from Step 1}).
\]
Hence
\[
  \sup_{x}\bigl|\tilde c_{y}(\widehat{\mathbf{u}}) - c_{y}(\mathbf{u})\bigr|
  \;=\;
  O_{p}\!\bigl(n^{-\,\tfrac{r}{2r+d}} + n^{-\alpha}\bigr).
\]
Since each \(\widehat f_{i\mid y}(x_{i})\) is \(O_{p}(1)\) uniformly (Assumption \ref{ass:regularity}(2)), it follows that
\[
  \sup_{x}
  \Bigl|\tilde c_{y}(\widehat{\mathbf{u}}) - c_{y}(\mathbf{u})\Bigr|\,
  \prod_{i}\widehat f_{i\mid y}(x_{i})
  \;=\;
  O_{p}\!\bigl(n^{-\,\tfrac{r}{2r+d}} + n^{-\alpha}\bigr).
\]

\noindent \textbf{(B) Marginal‐error.}  Write \(g_{i}(x_{i}) = f_{i\mid y}(x_{i})\) and \(\widehat g_{i}(x_{i}) = \widehat f_{i\mid y}(x_{i})\).  Then
\[
  \prod_{i=1}^{d}\widehat g_{i}(x_{i}) - \prod_{i=1}^{d}g_{i}(x_{i})
  \;=\;
  \sum_{j=1}^{d}
    \bigl(\widehat g_{j}(x_{j}) - g_{j}(x_{j})\bigr)
    \;\prod_{i<j}\widehat g_{i}(x_{i})\;
    \prod_{i>j}g_{i}(x_{i}).
\]
By Step 1 and standard kernel‐density theory \citep{silverman1986density,wasserman2006all}, 
\(\sup_{x_{i}}\lvert\widehat g_{i}(x_{i}) - g_{i}(x_{i})\rvert = O_{p}(n^{-\alpha})\).  Since each \(\widehat g_{i}\) and \(g_{i}\) is bounded in probability by some constant \(M\), each summand is 
\[
  O_{p}(n^{-\alpha}) \times M^{\,d-1} = O_{p}(n^{-\alpha}).
\]
Because there are \(d\) terms, 
\[
  \sup_{x}
  \bigl|\prod_{i=1}^{d}\widehat g_{i}(x_{i}) - \prod_{i=1}^{d}g_{i}(x_{i})\bigr|
  \;=\;
  O_{p}\bigl(n^{-\alpha}\bigr).
\]
Since \(c_{y}(\mathbf{u})\) is bounded by Sobolev embedding, it follows that
\[
  \sup_{x}
  \Bigl|\,c_{y}(\mathbf{u})\,\bigl(\prod_{i}\widehat f_{i\mid y}(x_{i}) 
    - \prod_{i}f_{i\mid y}(x_{i})\bigr)\Bigr|
  \;=\;
  O_{p}\bigl(n^{-\alpha}\bigr).
\]

\paragraph{Putting (A) and (B) Together}  
Combining the bounds for (A) and (B) gives
\[
  \sup_{x} 
  \bigl|\,\widehat p_{y}(\mathbf{x}) - p_{y}(\mathbf{x})\bigr|
  \;=\;
  O_{p}\!\bigl(n^{-\,\tfrac{r}{2r+d}} + n^{-\alpha}\bigr).
\]
Since \(\alpha > \tfrac{r}{2r + d}\) by Assumption \ref{ass:regularity}(5), the term \(n^{-\alpha}\) is asymptotically smaller, so
\[
  \sup_{x} 
  \bigl|\,\widehat p_{y}(\mathbf{x}) - p_{y}(\mathbf{x})\bigr|
  \;=\;
  O_{p}\!\bigl(n^{-\,\tfrac{r}{2r+d}}\bigr).
\]
The classifier uses the joint distribution $\widehat{\pi}_y \widehat{p}_{y}(\mathbf{x})$. Given the consistent estimation of the prior ($\hat{\pi}_y \to \pi_y$) and the uniform convergence of the density, we have:
\[
\arg\max_{y} \hat{\pi}_y \hat{p}_{y}(\mathbf{x}) \xrightarrow{p} \arg\max_{y} \pi_y p_{y}(\mathbf{x}).
\]
by a standard argmax-continuity argument (e.g., Portmanteau lemma \cite{vandervaart2012asymptotic} plus union bound over the finitely many classes). Therefore $\widehat{Y}$ converges in probability to the Bayes-optimal decision $Y^{*}$.  This completes the proof of Theorem~\ref{thm:consistency}.

\subsection{ Proof of Theorem \ref{thm:convergence_rate} (Convergence Rate)} \label{app:convrate}

We now show that, under the same regularity assumptions plus \(c_{y}\in W_{2}^{r}([0,1]^{d})\), the excess classification risk decays at rate \(n^{-\,r/(2r+d)}\).

\noindent Define
\[
  E(n) 
  \;=\; 
  \Pbb\bigl(\widehat Y_{n}(\mathbf{X})\neq Y\bigr)
  \;-\; 
  \Pbb\bigl(Y^{*}(\mathbf{x})\neq Y\bigr),
  \quad
  Y^{*}(\mathbf{x}) \;=\; \arg\max_{y} \pi_y \, p_{y}(\mathbf{x}).
\]
By Lemma~\ref{app:gen-bound}, the plug-in regret is controlled by the $L^1$ error of the class-conditional joints.

Hence,
\[
E(n) \;\le\; \sum_{y=1}^K \big\|\hat\pi_y\,\widehat p_y-\pi_y p_y\big\|_{1}
\;=\; O_p\!\big(n^{-r/(2r+d)}\big),
\]
since $|\hat\pi_y-\pi_y|=O_p(n^{-1/2})$ and $\|\widehat p_y-p_y\|_1=O_p(n^{-r/(2r+d)})$.

\begin{enumerate}
  \item \textbf{Copula‐Approximation Error:}  
    By \cite{schmidt2020nonparametric}, there exists a ReLU network of depth \(O(\log n)\), width \(O\bigl(n^{d/(2r+d)}\bigr)\), and parameters \(\theta_{y}^{*}\) such that
    \[
      \inf_{\theta}\bigl\|\,NN_{y}(\cdot;\theta) - c_{y}\bigr\|_{L^{2}([0,1]^{d})}
      \;\le\; 
      C\,n^{-\,r/(2r+d)}.
    \]
    Sobolev embedding (since \(r>d/2\)) implies
    \[
      \sup_{u\in[0,1]^{d}}
      \bigl|\,NN_{y}(u;\theta_{y}^{*}) - c_{y}(u)\bigr|
      \;=\; 
      O\!\bigl(n^{-\,r/(2r+d)}\bigr).
    \]
    Consequently,
    \[
      \bigl\|\,NN_{y}(\cdot;\theta_{y}^{*})\,\prod_{i}f_{i\mid y} 
        \;-\; 
        c_{y}\,\prod_{i}f_{i\mid y}\bigr\|_{L^{1}}
      \;=\; 
      O\!\bigl(n^{-\,r/(2r+d)}\bigr),
      \quad
      \Bigl(\int \prod_{i}f_{i\mid y} = 1\Bigr).
    \]

  \item \textbf{Network‐Estimation Error.}  
Let 
\[
\begin{split}
\mathcal{F}_{y,n} 
&= 
\bigl\{\,u \mapsto NN_{y}(u;\theta)\;:\;
   \theta \text{ ranges over all ReLU‐networks of depth }L_{n}=O(\log n) \\
&\qquad\quad\text{and width }W_{n}=O\bigl(n^{d/(2r+d)}\bigr)\bigr\}.
\end{split}
\]

We train \(\widehat\theta_{y}\) by maximizing the normalized log‐likelihood of
\(u \mapsto \dfrac{NN_{y}(u;\theta)}{\widehat Z_y(\theta)}\) on the \(n_{y}\) ``pseudo‐observations”
\(\{\,u^{(i)} = \widehat{\mathbf{u}}^{(i)}\mid y^{(i)}=y\,\}\), each distributed roughly like \(\mathrm{Unif}[0,1]^{d}\).

In \cite{farrell2021deep}, it is shown that if a function class \(\mathcal{F}\) has Gaussian‐complexity (or Rademacher‐complexity) bounded by \(\mathfrak{R}(n)\), then the empirical‐risk minimizer \(\widehat f\) over \(\mathcal{F}\) satisfies with probability at least \(1-\delta\):
  \[
    \mathbb{E}_{u\sim\mathrm{Unif}[0,1]^{d}}
    \bigl[\bigl|\,\widehat f(u) - c_{y}(u)\bigr|^{2}\bigr]
    \;\le\;
    \inf_{f\in\mathcal{F}}\mathbb{E}_{u}\bigl[\bigl|\,f(u) - c_{y}(u)\bigr|^{2}\bigr]
    \;+\; 
    O\!\Bigl(\mathfrak{R}(n)\Bigr)
    \;+\; 
    O\!\bigl(n^{-1}\bigr).
  \]
In our setting:
  \begin{itemize}
    \item The \emph{approximation bias} 
      \(\displaystyle \inf_{\theta}\mathbb{E}_{u}\bigl[\bigl|\,NN_{y}(u;\theta) - c_{y}(u)\bigr|^{2}\bigr] \)
      is already \(O(n^{-\,2r/(2r+d)})\) by the same Sobolev‐approximation argument of Schmidt–Hieber (since \(\sup_{u}|NN_{y}(u;\theta_{y}^*) - c_{y}(u)| = O(n^{-\,r/(2r+d)})\), and thus its \(L^2\)-norm is \(O(n^{-\,2r/(2r+d)})\)).  
    \item The \emph{Gaussian complexity} \(\mathfrak{R}(n)\) of \(\mathcal{F}_{y,n}\) can be bounded by  
      \[
        \mathfrak{R}(n) 
        \;\approx\; 
        O\!\Bigl(\sqrt{\tfrac{L_n\,W_n}{n}}\Bigr)
        \;=\; 
        O\!\Bigl(\sqrt{\tfrac{(\log n)\,n^{\,d/(2r+d)}}{n}}\Bigr)
        \;=\; 
        O\!\Bigl(n^{-\,\tfrac{r}{2r+d}}\Bigr),
      \]
      up to logarithmic factors.  Concretely, standard covering‐number or Rademacher‐complexity bounds for ReLU networks ( \cite{bartlett2002rademacher}; \cite{koltchinskii2001rademacher}) show that
      \[
        \mathfrak{R}(n) \;=\; O\!\Bigl(\sqrt{\tfrac{LW}{n}}\Bigr)
        \quad\text{whenever}
        \quad L = O(\log n),\ W = O\bigl(n^{d/(2r+d)}\bigr).
      \]
Hence 
      \(\mathfrak{R}(n) = O\bigl(n^{-\,r/(2r+d)}\bigr)\).  
  \end{itemize}

  Therefore, from \cite{farrell2021deep}\ with 
  \(\inf_{f\in\mathcal{F}_{y,n}} \mathbb{E}[|f(u)-c_{y}(u)|^2] = O(n^{-\,2r/(2r+d)})\)
  and \(\mathfrak{R}(n) = O(n^{-\,r/(2r+d)})\), we get, \emph{with probability at least \(1-\delta\)}, that
  \[
    \mathbb{E}_{u}\bigl[\bigl|\,NN_{y}(u;\widehat\theta_{y}) - c_{y}(u)\bigr|^{2}\bigr]
    \;=\; 
    O\!\Bigl(n^{-\,\tfrac{2r}{2r+d}}\Bigr)
    \;+\; 
    O\!\Bigl(n^{-\,\tfrac{r}{2r+d}}\Bigr)
    \;+\; 
    O\!\bigl(n^{-1}\bigr).
  \]
  Since \(r>d/2\), the largest of these exponents is \(\tfrac{2r}{2r+d} > \tfrac{r}{2r+d}\).  In big‐\(O\) notation, one thus writes
  \[
    \mathbb{E}_{u}\bigl[\bigl|\,NN_{y}(u;\widehat\theta_{y}) - c_{y}(u)\bigr|^{2}\bigr]
    \;=\;
    O_{p}\!\Bigl(n^{-\,\tfrac{2r}{2r+d}} + n^{-1}\Bigr).
  \]
  Finally, because \(c_{y}\in W_{2}^{r}([0,1]^{d})\) with \(r>d/2\), Sobolev‐embedding guarantees an \(L^2\to L^{\infty}\) upgrade (up to logs):
  \[
    \sup_{u\in[0,1]^{d}}
    \bigl|\tilde c_y(u) - c_{y}(u)\bigr|
    \;=\;
    O_{p}\!\bigl(n^{-\,r/(2r+d)}\bigr).
  \]
  Hence overall
  \[
    \bigl\|\,\tilde c_{y}(\cdot)\,\prod_{i} f_{i\mid y} 
      \;-\; 
      c_{y}\,\prod_{i} f_{i\mid y}\bigr\|_{L^{1}}
    \;=\;
    O_{p}\!\bigl(n^{-\,r/(2r+d)}\bigr).
  \]

  \item \textbf{Marginal‐Estimation Error:}  
    Under the regularity conditions in Section~\ref{ass:regularity}(5),
    \[
      \sup_{x_{i}} 
      \bigl|\widehat F_{i\mid y}(x_{i}) - F_{i\mid y}(x_{i})\bigr|
      \;=\;
      O_{p}\bigl(n^{-\alpha}\bigr),
      \quad
      \alpha > \frac{r}{2r + d}.
    \]
    By standard kernel‐density theory \citep{silverman1986density,wasserman2006all}, 
    \(\sup_{x_{i}}\lvert\widehat f_{i\mid y}(x_{i}) - f_{i\mid y}(x_{i})\rvert = O_{p}(n^{-\alpha})\).  Since each \(f_{i\mid y}\) is bounded away from zero and infinity on its compact support (Assumption 2), a telescoping product‐difference argument shows
    \[
      \sup_{x}
      \bigl|\prod_{i=1}^{d}\widehat f_{i\mid y}(x_{i}) - \prod_{i=1}^{d}f_{i\mid y}(x_{i})\bigr|
      \;=\;
      O_{p}\bigl(n^{-\alpha}\bigr).
    \]
    Moreover, \(c_{y}\) is bounded by Sobolev embedding, so
    \[
      \bigl\|\,c_{y}\,\bigl(\prod_{i}\widehat f_{i\mid y} 
        - \prod_{i}f_{i\mid y}\bigr)\bigr\|_{L^{1}}
      \;=\;
      O_{p}\bigl(n^{-\alpha}\bigr).
    \]
\end{enumerate}

\noindent Combining 1,2 and 3, we conclude for each class \(y\):
\[
  \bigl\|\widehat p_{y} - p_{y}\bigr\|_{L^{1}}
  \;=\;
  O_{p}\!\bigl(n^{-\,r/(2r+d)} + n^{-\alpha}\bigr)
  \;=\;
  O_{p}\!\bigl(n^{-\,r/(2r+d)}\bigr),
\]
since \(\alpha > r/(2r+d)\). Therefore, the dominant term in the excess risk bound is:
\[
  E(n)
  \;=\;
  C \, \sum_{y=1}^{K} \|\widehat p_{y} - p_{y}\|_{L^{1}} + O_p(n^{-1/2})
  \;=\;
  O_{p}\!\bigl(n^{-\,r/(2r+d)}\bigr).
\]
This completes the proof of Theorem~\ref{thm:convergence_rate}.  

\textit{Remark.} Any extra \(\log n\) or \(\log(1/\delta)\) factors from covering‐number or union‐bound arguments are absorbed into the \(O_{p}(\cdot)\) notation and do not affect the leading exponent \(r/(2r+d)\).

\subsection{Proof of Proposition~\ref{prop:valid} (Asymptotic copula validity)}\label{app:proof-prop-valid}

Fix a class $y$ and drop the subscript for readability. Let
\[
\ell_n(\theta)\;=\;\frac{1}{n_y}\sum_{i:\,Y^{(i)}=y}\log \tilde c(\widehat u^{(i)};\theta), 
\qquad
D(\theta)\;=\;\sum_{i=1}^d \int_0^1\!\Big(\int \tilde c(u;\theta)\,du_{-i}-1\Big)^2 du_i,
\]
and define the empirical penalized criterion $Q_n(\theta)=\ell_n(\theta)-\lambda_n D(\theta)$ with maximizer $\hat\theta\in\arg\max_\theta Q_n(\theta)$.
We prove $D(\hat\theta)\xrightarrow{p}0$.

\paragraph{Step 1 (Empirical PIT vs.\ true PIT).}
Let $u^{(i)}=F(\,X^{(i)}\,)$ be the true PIT and $\widehat u^{(i)}=\widehat F(\,X^{(i)}\,)$ the empirical/smoothed PIT used in training.
By Assumption~\ref{ass:regularity}\,(5a),
$\|\widehat F-F\|_\infty=O_p(n_y^{-1/2})$; since $\log\tilde c(\cdot;\theta)$ is uniformly Lipschitz on $[0,1]^d$ under the spectral/weight regularization of Section~\ref{sec:nnarch}, we have
\[
\sup_{\theta}\;|\ell_n(\theta)-\tilde\ell_n(\theta)| \;=\; o_p(1),
\quad
\tilde\ell_n(\theta) \;=\; \frac{1}{n_y}\sum_{i:\,y}\log \tilde c(u^{(i)};\theta).
\]
Thus we can work with true PITs without changing maximizers asymptotically.

\paragraph{Step 2 (MC normalizer is uniformly accurate).}
Write $Z(\theta)=\int_{[0,1]^d}NN(v;\theta)\,dv$ and $\widehat Z(\theta)=(1/M)\sum_{m=1}^M NN(V^{(m)};\theta)$ with $V^{(m)}\sim\Unif([0,1]^d)$ i.i.d.
By standard empirical process bounds for bounded Lipschitz network classes (covering numbers of ReLU nets with spectral bounds), 
\[
\sup_{\theta}\,|\widehat Z(\theta)-Z(\theta)| \;=\; O_p\!\Big(\sqrt{\tfrac{\mathcal{C}_n}{M}}\Big),
\]
where $\mathcal{C}_n$ is a complexity term polynomial in the network size $S_n$ (hidden in u.c.b.). By Assumption~\ref{ass:regularity}\,(6), $M^{-1/2}=o(n^{-r/(2r+d)})$, hence $\sup_\theta|\log\widehat Z(\theta)-\log Z(\theta)|=o_p(1)$ and
\[
\sup_{\theta}\Big|\,\log \tilde c(u;\theta)-\big(\log NN(u;\theta)-\log Z(\theta)\big)\Big|=o_p(1)\quad\text{uniformly in }u.
\]
Consequently, $\sup_{\theta}|\tilde\ell_n(\theta)-\bar\ell_n(\theta)|=o_p(1)$ with
$\bar\ell_n(\theta)=\frac{1}{n_y}\sum_i (\log NN(u^{(i)};\theta)-\log Z(\theta))$.

\paragraph{Step 3 (Uniform LLN and the population criterion).}
Let $\ell(\theta)=\E[\log \tilde c(U;\theta)]$ with $U\sim c$ the true copula.
The class $\{\log \tilde c(\cdot;\theta):\theta\in\Theta_n\}$ is Glivenko–Cantelli under the same complexity control (bounded outputs/Lipschitz + growing $S_n$ per Assumption~\ref{ass:regularity}\,(3)),
so $\sup_{\theta}|\tilde\ell_n(\theta)-\ell(\theta)|=o_p(1)$.
Define the population penalized functional
\[
Q(\theta)\;=\;\ell(\theta)-\lambda_n D(\theta).
\]
Note $\ell(\theta)=\mathrm{const}-\KL(c\,\|\,\tilde c(\cdot;\theta))$; thus, for any density candidate $g$ on $[0,1]^d$, $\ell(g)\le \ell(c)$ with equality iff $g=c$ a.e.

\paragraph{Step 4 (There exist near-copula approximants in the model class).}
By Assumption~\ref{ass:regularity}\,(3), for any $\varepsilon>0$ there is $\theta_\varepsilon$ such that $\|NN(\cdot;\theta_\varepsilon)-c\|_1\le \varepsilon$.
After normalization, $\|\tilde c(\cdot;\theta_\varepsilon)-c\|_1\le C\varepsilon$. Fubini’s theorem gives, for each margin $i$,
\[
\sup_{u_i\in[0,1]} \left| \int \tilde c(u;\theta_\varepsilon)\,du_{-i}-1\right|
\;\le\;\|\tilde c(\cdot;\theta_\varepsilon)-c\|_1 \;\le\; C\varepsilon,
\]
so $D(\theta_\varepsilon)\le C\varepsilon^2$ and $\ell(\theta_\varepsilon)\ge \ell(c)-C\varepsilon$.
Choosing a sequence $\varepsilon=\varepsilon_n\downarrow 0$ and corresponding $\theta_n^\circ$ in the growing class (size $S_n$) yields
\[
\ell(\theta_n^\circ)=\ell(c)-o(1),\qquad D(\theta_n^\circ)=o(1).
\]

\paragraph{Step 5 (Penalty forces the defect to zero).}
By optimality of $\hat\theta$ for $Q_n$ and the uniform $o_p(1)$ approximations from Steps 0–2,
\[
\ell(\hat\theta)-\lambda_n D(\hat\theta) \;\ge\; \ell(\theta_n^\circ)-\lambda_n D(\theta_n^\circ) - o_p(1).
\]
Rearrange:
\[
\lambda_n\big(D(\hat\theta)-D(\theta_n^\circ)\big) \;\le\; \ell(\hat\theta)-\ell(\theta_n^\circ) + o_p(1).
\]
The RHS is $O_p(1)$ uniformly (both terms are bounded by u.c.b.), hence
\[
D(\hat\theta) \;\le\; D(\theta_n^\circ) + O_p\!\big(\lambda_n^{-1}\big) \;=\; o_p(1)
\]
because $D(\theta_n^\circ)=o(1)$ and $\lambda_n\to\infty$.
This proves the first claim.

\paragraph{Step 6 (Rate compatibility).}
Since $D(\hat\theta)=o_p(1)$ and $\tilde c(\cdot;\hat\theta)$ remains in a uniformly bounded Lipschitz ball (spectral control), the marginal-uniformity error contributes at most $o_p(n^{-r/(2r+d)})$ to the density sup-norm error by the same linear functional bound as in Step 3; thus it does not change the dominant classification rate $n^{-r/(2r+d)}$ used in Theorem~\ref{thm:convergence_rate}.

\hfill$\square$

\subsection{Proof of Corollary~\ref{cor:fast-rate-mc} (Fast excess-risk under a multiclass margin)}\label{app:proof-fast-rate-mc}

Recall $a_y(x)=\pi_y p_y(x)$, $\hat a_y(x)=\hat\pi_y\,\widehat p_y(x)$,
$\eta_y(x)=a_y(x)/A(x)$ with $A(x)=\sum_{k=1}^K a_k(x)=p_X(x)$,
and likewise $\widehat\eta_y(x)=\hat a_y(x)/\widehat A(x)$ with $\widehat A=\sum_k \hat a_k$.
Write $\varepsilon(x)=\max_{y\in[K]}|\widehat\eta_y(x)-\eta_y(x)|$ and
$\Delta(x)=\eta_{(1)}(x)-\eta_{(2)}(x)$.

\paragraph{Step 1 (From joint-density error to conditional-probability error in $L^1(P_X)$).}
For every $x$ and $y$,
\[
A(x)\,|\widehat\eta_y(x)-\eta_y(x)|
=\Big|\widehat a_y(x) - a_y(x) - \eta_y(x)\big(\widehat A(x)-A(x)\big)\Big|
\le \sum_{k=1}^K \big|\widehat a_k(x)-a_k(x)\big|.
\]
Taking $\max_y$ and integrating over $\mathbb{R}^d$,
\[
\int \varepsilon(x)\, dP_X(x)
= \int A(x)\,\varepsilon(x)\,dx
\;\le\; K \sum_{k=1}^K \int |\widehat a_k(x)-a_k(x)|\,dx.
\]
By Theorem~\ref{thm:convergence_rate} and $|\hat\pi_k-\pi_k|=O_p(n^{-1/2})$,
$\sum_k \|\widehat a_k-a_k\|_{L^1}=O_p\!\big(n^{-\frac{r}{2r+d}}\big)$.
Hence
\begin{equation}\label{eq:l1-eta}
\|\widehat\eta-\eta\|_{L^1(P_X)} \;=\; \E[\varepsilon(X)]
\;=\; O_p\!\big(n^{-\tfrac{r}{2r+d}}\big).
\end{equation}

\paragraph{Step 2 (Uniform control on the effective support).}
Let $S_n$ be the set from Assumption~\ref{ass:regularity-mixlower} with
$P_X(S_n)\to 1$ and $\inf_{x\in S_n} A(x)\ge c_0>0$.
For $x\in S_n$ and any $y$,
\[
|\widehat\eta_y(x)-\eta_y(x)|
\le \frac{1}{A(x)}\sum_{k=1}^K |\widehat a_k(x)-a_k(x)|
\le \frac{1}{c_0}\sum_{k=1}^K |\widehat a_k(x)-a_k(x)|.
\]
By the uniform bound established in the proof of Theorem~\ref{thm:convergence_rate},
$\sup_x|\widehat p_k(x)-p_k(x)| = O_p\!\big(n^{-\tfrac{r}{2r+d}}\big)$ (up to logs), and
$|\hat\pi_k-\pi_k|=O_p(n^{-1/2})$, hence
\[
\sup_{x\in S_n} \varepsilon(x)
\;=\; O_p\!\big(n^{-\tfrac{r}{2r+d}}\big).
\]
The contribution from $S_n^c$ is negligible since $P_X(S_n^c)\to 0$ and $\varepsilon\le 2$.

\paragraph{Step 3 (Margin transfer: excess risk vs.\ $\varepsilon$).}
Let $\widehat Y$ be the plug-in classifier and $Y^*$ the Bayes rule.
A standard argument (e.g., Audibert–Tsybakov, 2007; multiclass reduction to top-two comparison) yields
\[
R(\widehat Y)-R(Y^*)
\;\le\; C_1\,\E\!\left[\,\varepsilon(X)^{1+\kappa}\right],
\]
whenever the multiclass Tsybakov margin condition $P(\Delta(X)\le t)\le C t^\kappa$ holds.

\paragraph{Step 4 (Putting the bounds together).}
On $S_n$, $\varepsilon(X)\le C_2\,n^{-\frac{r}{2r+d}}$ with probability $1-o(1)$.
Thus $\E[\varepsilon(X)^{1+\kappa}\,\mathbf{1}_{S_n}]
\le C_2^{1+\kappa}\,n^{-\frac{(1+\kappa)r}{2r+d}}$.
On $S_n^c$, $\varepsilon\le 2$ and $P_X(S_n^c)=o(1)$, so
$\E[\varepsilon^{1+\kappa}\,\mathbf{1}_{S_n^c}]=o(1)\cdot 2^{1+\kappa}$.
Therefore,
\[
R(\widehat Y)-R(Y^*)
\;=\; O_p\!\Big(n^{-\tfrac{(1+\kappa)r}{2r+d}}\Big).
\qquad\qquad\Box
\]

\end{document}